\documentclass[11pt]{article}
\usepackage{amsfonts}
\usepackage{amsmath}
\usepackage{amssymb}
\usepackage{graphicx}
\usepackage{comment}
\usepackage{algorithm}
\usepackage{algorithmic}
\usepackage{cases}
\usepackage{geometry, tabularx}
\usepackage{multirow,makecell}
\usepackage{enumerate}
\usepackage{bm}
\usepackage{diagbox}
\usepackage{pifont}
\usepackage{footnote}
\usepackage{titlesec}
\usepackage[titletoc]{appendix}
\usepackage{threeparttable}
\usepackage{booktabs}
\usepackage{setspace}
\usepackage{natbib}
\usepackage{mathrsfs}
\usepackage{makecell}
\usepackage{graphicx,color}
\usepackage[bf,SL,BF]{subfigure}
\usepackage{booktabs}

\providecommand{\U}[1]{\protect\rule{.1in}{.1in}}

\newtheorem{thm}{\bf Theorem}      
\newtheorem{cor}{\bf Corollary}[section]     
\newtheorem{lem}{\bf Lemma}
\newtheorem{prop}{\bf Proposition}

\usepackage{hyperref}
\hypersetup{hypertex=true,
            colorlinks=true,
            linkcolor=blue,
            anchorcolor=green,
            citecolor=blue}

\newenvironment{proof}[1][Proof]{\noindent\textbf{#1.} }{\ \rule{0.5em}{0.5em}}
\normalsize
\titlespacing*{\section} {0pt}{9pt}{0pt}
\numberwithin{equation}{section}
\allowdisplaybreaks

\begin{document}
\normalsize
\title{Parallel Algorithms for Structured Sparse Support Vector Machines: Application in Music Genre Classification}

\author{Rongmei Liang\thanks{Department of Statistics and Data Science, Southern University of Science and Technology, Guangdong, 518055, P.R. China. Email: liang\_r\_m@163.com}, \\
Zizheng Liu\thanks{Liu and Liang contributed equally to this paper. Faculty of Education, Languages, Psychology \& Music, SEGi University, Kuala Lumpur, 47810, Malaysia. Email: liuzizheng199608@163.com},\ \ 
Xiaofei Wu\thanks{College of Mathematics and Statistics, Chongqing University, Chongqing 401331, P.R. China. E-mail: xfwu1016@163.com}, \\  
Jingwen Tu\thanks{Corresponding author. School of Mathematical and Physical Sciences, Chongqing University of Science and Technology, Chongqing 401331, P.R. China. Email: 15310293521@163.com}}
\date{}
\maketitle
\vspace{-.40in}

\begin{abstract}
Mathematical modelling, particularly through approaches such as structured sparse support vector machines (SS-SVM), plays a crucial role in processing data with complex feature structures, yet efficient algorithms for distributed large-scale data remain lacking. To address this gap, this paper proposes a unified optimization framework based on a consensus structure. This framework is not only applicable to various loss functions and combined regularization terms but can also be effectively extended to non-convex regularizers, demonstrating strong scalability. Building upon this framework, we develop a distributed parallel alternating direction method of multipliers (ADMM) algorithm to efficiently solve SS-SVMs under distributed data storage. To ensure convergence, we incorporate a Gaussian back-substitution technique. Additionally, for completeness, we introduce a family of sparse group Lasso support vector machine (SGL-SVM) and apply it to music information retrieval. Theoretical analysis confirms that the computational complexity of the proposed algorithm is independent of the choice of regularization terms and loss functions, underscoring the universality of the parallel approach. Experiments on both synthetic and real-world music archive datasets validate the reliability, stability, and efficiency of our algorithm. 
\end{abstract}


\section{Introduction}\label{sec1}
\quad \  Support vector machine (SVM \cite{Vapnik1995TheNO}) stands as a prominent supervised machine learning technique, which is widely employed for classification and regression tasks \cite{Wu2025MultiLA,Wu2025ParallelAA}.
In modern data science, high-dimensional classification problems are ubiquitous, where the number of feature dimensions often far exceeds the number of samples. To build accurate and interpretable models in such scenarios, sparse regularized SVM such as $\ell_1$-norm SVM \cite{Zhu20031SV}, the smoothly clipped absolute deviation (SCAD) penalized SVM \cite{Becker2011ElasticSA} and minimax concave penalty (MCP) penalized SVM \cite{Zhang2016VariableSF} are widely considered. However, single sparse regularization terms fail to effectively utilize the prior information of the data, resulting in sub-optimal performance when dealing with data featuring complex structures.

To address these limitations, many papers aim to structured sparse regularization terms with SVM, which include the elastic net SVM (EN-SVM, \cite{Wang2006TheDR}), sparse fused Lasso SVM (SFL-SVM, \cite{Tibshirani2005SparsityAS}). The aforementioned structured sparse SVM (SS-SVM) models can fully utilize available prior information to enhance model performance. For example, 
EN-SVM can utilize the correlation information among features to retain or eliminate the coefficients of relevant features simultaneously. When simply applying $\ell_1$ regularization, due to the correlation between features, one feature may be randomly selected, whereas the elastic net \cite{Zou2005RegularizationAV} avoids this issue.
SFL-SVM is primarily designed for handling data with sequential structures. It can make use of this sequential characteristic to ensure that the coefficients of features at adjacent time points are as close as possible, thereby effectively capturing the changing trends and structural information in the data.
When we have prior knowledge about the grouping of data, the sparse group Lasso \cite{Simon2013ASL} can utilize this information to conduct feature selection on a feature-group basis. It can not only perform sparse processing on the features within a group, setting the coefficients of unimportant features within the group to zero, but also select groups, setting the coefficients of an entire group to zero, thus improving the interpretability and performance of the model. 
These characteristics are particularly important for high-dimensional classification problems. This paper proposes a family of sparse group Lasso SVM (SGL-SVM) by incorporating a sparse group Lasso regularization term with various loss functions as discussed by \cite{Liang2024LinearizedAD}, which can serve as a significant complementary instance within the realm of SS-SVMs.

The field of music information retrieval (MIR) aptly demonstrates the significance and challenges associated with SS-SVM in practical applications \cite{Chen2025Investigating}. Automated music genre classification represents not only a canonical machine-learning task but also a means to understand the intricate internal composition of musical styles. When representing music data, we encounter multiple distinct characteristics that necessitate specialized modeling approaches.
First, music features often exhibit correlations. For instance, the timbre of an instrument can be related to the harmony it produces. A particular type of timbre might be more commonly associated with specific harmonic progressions. In such instances, we need a model that can not only select relevant features but also account for these correlations. The elastic net-like mechanism in SS-SVMs helps to jointly shrink correlated features, producing a more stable and interpretable model.
Second, music data has an intrinsic sequential nature. Musical compositions unfold over time, and the order of events is of great importance. The sparse fused lasso can handle the sequential structure in data, encouraging neighboring features to have similar values. In the context of music, it can capture the smooth transitions between different musical phrases or segments, and help the model identify and utilize these sequential relationships for more accurate genre classification.
Finally, music data inherently possesses a group-based structure. By extracting a multitude of features that describe rhythm, harmony, timbre, and dynamics from audio, we can group these features into different categories. For example, all rhythm-related features form one group, while harmony-related features form another. To generate interpretable conclusions, we require a model capable of feature selection at the group level.

With the rapid growth of data scale, a single machine is generally unable to store all the data. As a result, data often has to be stored in a distributed manner. However, when dealing with data stored in a distributed system, efficiently and stably solving optimization problems involving complex non-smooth regularization terms has become a key computational bottleneck \cite{Wu2025ParallelAA, Wu2025AUC}. Existing general solvers frequently struggle to meet the requirements of large-scale, high-dimensional, structured sparsity, and distributed computing. Therefore, SS-SVM, which integrates the capabilities of sparse group Lasso, elastic net, and sparse fused lasso, is of crucial importance in the field of music genre classification. It can effectively address the group structure, feature correlations, and sequential order in music data, enabling more precise, interpretable, and efficient classification of musical genres. Although in the scenario of non-distributed data storage, there already exist multiple optimization algorithms for SVMs with specific regularization terms \cite{Liang2024LinearizedAD, Guan2020AnEP}, a general framework that can uniformly handle multiple sparse patterns and their non-convex variants, while having strict theoretical guarantees and efficient parallel capabilities, is still lacking. The absence of such a framework forces researchers and application developers to customize different solution schemes for different models, and they also face difficulties in efficiency and scalability when dealing with large-scale data. 

To fill this gap, this paper proposes a novel unified optimization framework and its corresponding parallel and distributed algorithm. Our framework can elegantly encapsulate up to \textbf{54} SS-SVM models, including the EN-SVM, SFL-SVM, SGL-SVM, and their non-convex generalizations. We have designed an efficient solution algorithm for this framework, which natively supports distributed data storage and computation and is highly suitable for handling large-scale datasets. More importantly, we provide a complete convergence analysis for this algorithm and prove that it achieves an improved sublinear convergence rate under mild conditions, which provides a solid mathematical foundation for the reliability and efficiency of model solving. To empirically validate the superior performance and practical utility of the proposed framework and algorithm, we apply them to interpretable music genre analysis. The availability of large-scale music datasets such as FMA (free music archive) enables us to fully demonstrate the parallel computing advantages of the algorithm in a distributed environment. And the analysis results can be directly translated into quantifiable musicological insights, meeting the interpretability requirements of music education, musicological research, and music recommendation systems.

The structure of the remaining part of this paper is as follows: Section \ref{sec2} reviews the relevant work; Section \ref{sec3} elaborates on the unified optimization framework and parallel solution algorithm we proposed and presents the theoretical analysis of the algorithm's convergence; Section \ref{sec4} demonstrates how to apply the framework to music genre analysis and shows the corresponding experimental results, including the comparison of algorithm performance and the musicological findings extracted from the model; finally, Section \ref{sec5} concludes the paper and provides an outlook for future work. 
\section{Preliminaries}\label{sec2}
The purpose of this section is to sort out the research progress in three fields closely related to this paper, laying a foundation for the subsequent proposal of new methods and applications. First, we review a series of models that combine support vector machines with structured sparse regularization terms, especially the elastic net, sparse fused Lasso and sparse group Lasso. These models form the core components of the unified framework in this paper. Second, we introduce the alternating direction method of multipliers (ADMM) widely used in distributed optimization and its consensus form, which serves as the cornerstone for designing the parallel solution algorithm in this paper. Finally, we focus on the field of music information retrieval and systematically summarize the research context of the genre classification task. 
\subsection{Structured sparse regularization terms}\label{sec2.1}
In this section, we will systematically introduce the combination of the elastic net, sparse fused Lasso, sparse group Lasso and their non-convex variants with the SVM model.
\subsubsection{Convex regularization model}\label{sec2.1.1}
 Given a data set $\{\bm{X}, \bm{y}\} = \{(\bm{x}_i, y_i)\}_{i=1}^n$, where $\bm{x}_i \in \mathbb{R}^p$ and $y_i \in \{-1,+1\}$, the traditional SVM can be expressed as:
\begin{equation}
    \label{traditional-SVM}
    \mathop {\min }\limits_{\bm{\beta}, \beta_0} \frac{1}{n}  \sum_{i=1}^{n}[1-y_i(\bm{x}_i^T \bm{\beta}  + \beta_0)]_{+} + \frac{\lambda}{2}\|\bm{\beta}\|_2^2, 
\end{equation}
where $\lambda$ is tuning parameter, the $(\bm{\beta}, \beta_0)$ pair is the decision variable with $\bm{\beta} \in \mathbb{R}^p$ and $\beta_0 \in \mathbb{R}$, $[1-y_i(\bm{x}_i^T \beta_j  + \beta_0)]_{+} = \max\{1-y_i(\bm{x}_i^T \beta_j  + \beta_0), 0\}$ is the hinge loss function.
SVM improves generalization performance by maximizing the classification margin. However, their decision functions rely on all features and lack the ability for feature selection. The introduction of $\ell_1$ regularization (Lasso, \cite{Tibshirani1996RegressionSA}) enables SVMs to generate sparse solutions, where the weights of some features are exactly zero, thus achieving embedded feature selection. Nevertheless, standard Lasso has limitations in terms of unstable selection when dealing with highly correlated features and fails to utilize the structured prior information among features. 

To address these issues, Wang et al. \cite{Wang2006TheDR}  proposed the EN-SVM, which combines $\ell_1$ and $\ell_2$ regularization.
\begin{equation}
    \label{en-SVM}
    \mathop {\min }\limits_{\bm{\beta}, \beta_0} \frac{1}{n}  \sum_{i=1}^{n}[1-y_i(\bm{x}_i^T \bm{\beta}  + \beta_0)]_{+} + \lambda_1\|\bm{\beta}\|_1 + \lambda_2\|\bm{\beta}\|_2^2, 
\end{equation}
where the $\ell_2$ term enhances the model's stability for multicollinear features, while the $\ell_1$ term maintains the feature selection ability. When dealing with high-dimensional and correlated features, the elastic-net shows better prediction performance and stability compared to the single Lasso. 

For audio features extracted frame by frame, when these features have a sequential structure, the total variation captures smoothness by fusing the weight differences between adjacent features:
\begin{align}
\label{TV1}
    \mathcal{\bm R}_{TV}(\bm \beta) = \lambda\sum_{j=2}^p|\beta_j - \beta_{j-1}|=\lambda\|\bm{F\beta}\|_1.
\end{align}
where $\bm{F}$ is a $(p-1) \times p$ matrix with all elements being 0, except for 1 on the diagonal and -1 on the superdiagonal.

Combining the total variation with $\ell_1$, namely the sparse fused Lasso, can achieve sparsification while ensuring the smoothness of the weights of adjacent features. This regularization term is particularly suitable for analyzing the patterns of music feature evolution over time, such as the smooth transition of melody contours or the gradual change of rhythm intensity. The SFL-SVM was initially introduced by Tibshirani et.al. \cite{Tibshirani2005SparsityAS} and is expressed as 
\begin{equation}
    \label{sfl-SVM}
    \mathop {\min }\limits_{\bm{\beta}, \beta_0} \frac{1}{n}  \sum_{i=1}^{n}[1-y_i(\bm{x}_i^T \bm{\beta} + \beta_0)]_{+} + \lambda_1\|\bm{\beta}\|_1 + \lambda_2\|\bm{F\beta}\|_1, 
\end{equation}

For features with a natural group structure, such as the rhythm group and harmony group in musical features, group Lasso achieves group-level feature selection by penalizing the $\ell_2$ norm of the weights of each feature group.
\begin{equation}
\label{GL}
\mathcal{\bm R}_{GL}(\bm \beta)=\lambda \sum_{m=1}^M\|\bm{\beta}_{(m)}\|_2 = \lambda \|\bm{\beta}\|_{2,1}
\end{equation}
This regularization term tends to either retain or eliminate an entire group of features simultaneously, making it particularly suitable for application scenarios where classification decisions need to be interpreted from macro-dimensions such as ``rhythmic complexity" and ``harmonic richness". 

Sparse group Lasso further combines the $\ell_1$ penalty and the group Lasso penalty to induce sparsity both between and within groups. To the best of our knowledge, the current research on SGL-SVM is limited to the two-step classification method proposed by Huo et al. \cite{Huo2020SGLAN}. To ensure the completeness of the work, we present the expression of this model as follows.
\begin{equation}
    \label{sgl-SVM}
    \mathop {\min }\limits_{\bm{\beta}, \beta_0} \frac{1}{n}  \sum_{i=1}^{n}[1-y_i(\bm{x}_i^T \bm{\beta}  + \beta_0)]_{+} + \lambda_1\|\bm{\beta}\|_1 + \lambda_2\|\bm{\beta}\|_{2,1}, 
\end{equation}

\subsubsection{Non-convex regularization generalization}\label{sec2.1.2}
Although convex regularization terms have good optimization properties, they may cause excessive shrinkage bias when estimating coefficients, which affects the accuracy of feature selection. Non-convex regularization terms alleviate this problem by reducing the penalty on large coefficients. Among them, SCAD \cite{Becker2011ElasticSA}  and MCP \cite{Zhang2016VariableSF} are the most classic ones. 

The MCP regularization is defined as follows:
\begin{equation}
    \label{mcp}
    \mathcal{\bm R}_{\lambda}(|\bm \beta_j|) = \begin{cases}
\lambda |\bm \beta_j|-\frac{\bm {\beta_j}^2}{2a}, & \text{if} \ \ |\bm \beta_j|  \le a\lambda\\[3mm]
\frac{a\lambda^2}{2}, & \text{if} \ \ |\bm \beta_j| > a\lambda
\end{cases}
\end{equation}
where $a>0$ is a constant. When \(|\bm \beta_j|\) is relatively small, MCP behaves like an $\ell_1$ penalty. As \(|\bm \beta_j|\) increases, the penalty gradually decreases to a constant, thus reducing the estimation bias. 

The SCAD regularization is defined as follows:
\begin{equation}
    \label{scad}
    \mathcal{\bm R}_{\lambda}(|\bm \beta_j|) = \begin{cases}
\lambda |\bm \beta_j|, & \text{if} \ \ |\bm \beta_j|  \le \lambda\\[3mm]
\frac{-\bm\beta_j^2+2a\lambda|\bm\beta_j|-\lambda^2}{2(a-1)}, & \text{if} \ \ \lambda < |\bm \beta_j|  \le a\lambda  \\[3mm]
\frac{(a+1)\lambda^2}{2}, & \text{if} \ \ |\bm \beta_j| > a\lambda 
\end{cases}
\end{equation}
where $a>2$ is a constant, and SCAD acts as an $\ell_1$ penalty when \(|\bm \beta_j| \leq \lambda\). When \(|\bm \beta_j|>a\lambda\), the penalty becomes a constant, with a smooth transition in between. 

By replacing the $\ell_1$ part in the convex regularization terms in Section \ref{sec2.1.1} with MCP or SCAD, the corresponding non-convex variant models can be obtained. These models theoretically possess the Oracle property \cite{Zhang2016VariableSF}, that is, when the sample size approaches infinity, they can select the true model with a probability of 1. However, the non-convexity makes the optimization problem more complex, and specially designed algorithms are needed to ensure convergence to a valid solution.

Most of the existing solution methods are designed for specific SVM models, and there is a lack of a general framework that can uniformly handle various convex/non-convex regularization terms. In addition, the demand for distributed parallel solution for large-scale data further increases the complexity of algorithm design. This paper aims to fill this gap by proposing a unified optimization framework and the corresponding parallel algorithm. 
\subsubsection{Overview of existing algorithms}\label{sec2.1.3}
For the optimization problem of specif SS-SVM, researchers have designed a variety of algorithms. The core difference lies in how to handle the smooth part of the loss function and the non-smooth part of the regularization term. According to different optimization principles, existing algorithms can be mainly divided into the following categories: coordinate descent-based methods, proximal gradient-based methods, and splitting methods based on the ADMM. In this subsection, relevant research will be sorted out according to this algorithm classification framework. 

$\bullet$ \textbf{Algorithms based on coordinate descent and its variants} approximate the optimal solution through iterative updates for each coordinate. They are widely popular due to their simplicity and high efficiency. For example, Yang and Zou \cite{Yang2013AnEA} proposed a generalized coordinate descent algorithm to solve a class of EN-SVM. These methods usually need to utilize the second-order information of the loss function or specific structures to improve efficiency. However, when dealing with regularization terms with adjacent coordinate coupling constraints such as fused Lasso, their update steps become complex, and parallelization also faces challenges. 

$\bullet$ \textbf{Algorithms based on the proximal gradient} decompose the problem into the gradient descent of the smooth part (the loss function) and the solution of the proximal operator of the non-smooth part (the regularization term). The form of the proximal operator determines the specific implementation and efficiency of the algorithm. For example, Xu et al. \cite{Xu2015ProximalGM} explored using the proximal gradient method to solve the huberized EN-SVM, while Zhu et al. \cite{Zhu2020SupportVM} applied it to solve the huberized Pinball EN-SVM. For the $\ell_1$ norm and the elastic net, their proximal operators have closed-form solutions. For the group Lasso, its proximal operator is block soft thresholding. However, for the fused Lasso, its proximal operator requires a more complex dynamic programming algorithm to solve. The convergence speed of such algorithms is usually sub-linear.

$\bullet$ \textbf{Algorithms based on ADMM}, decompose complex coupled optimization problems into a series of simpler sub-problems that can be solved in parallel or sequentially by introducing auxiliary variables. This method shows unique advantages when dealing with composite regularization terms or non-separable constraints. For example, Liang et al. \cite{Liang2024LinearizedAD} introduced the linearized ADMM to solve the EN-SVM. For fusion-type problems, Ye and Xie \cite{Ye2011SplitBM} designed an iterative algorithm based on the split Bregman method to solve large-scale generalized fused Lasso problems. Recently, Wu et al. \cite{Wu2025MultiLA,Wu2024MultiAD}  proposed a unified multi-block linearized ADMM framework for solving the quantile sparse fused Lasso and the pinball SFL-SVM. ADMM has good scalability and parallel potential, and our work is also carried out based on it.

In addition, there are also path algorithms such as the least angle regression (LAR) and quadratic program. For example, while proposing the EN-SVM, Wang et al. \cite{Wang2006TheDR} presented an LAR-based algorithm to calculate its complete regularization path; and Tibshirani et al.  \cite{Tibshirani2005SparsityAS} incorporated quadratic programming techniques for SFL-SVM. These methods can provide a global perspective for model selection, but they incur high computational costs in ultra-high dimensional or big-data scenarios. 

The existing work mainly has four limitations:
Firstly, there is a lack of algorithms for solving the SGL-SVM model. As far as we know, there is currently no algorithm specifically designed for solving the SGL-SVM model.
Secondly, the algorithms are overly specific. Most algorithms are tailored for specific regularization terms, lacking a general solution framework that can uniformly handle the elastic-net, the sparse group Lasso, the sparse fused Lasso, and their non-convex generalizations.
Thirdly, there is insufficient support for parallel computing. Many algorithms are inherently serial. Even if some have parallel implementations, they do not fully take into account the modern computing environment with distributed data storage.
Fourthly, the theoretical analysis is not unified. Especially for support vector machines with non-convex regularization terms, the analysis of the convergence and convergence rate of the algorithms often needs to be carried out separately for each variant, lacking a general theoretical tool. 

Table \ref{tab11} summarizes and compares the characteristics of the above-mentioned main algorithms and their support for various regularization models from the dimension of algorithm principles.
\begin{table}[h]
    \centering
    \resizebox{1\columnwidth}{!}{
    \begin{threeparttable}
    \caption{Comparison of main solving algorithms for SS-SVM}
    \label{tab11}
    \begin{tabular}{lccc}
    \hline
        Algorithm category & Supported regularization models & Parallelism potential & Difficulty of non-convex extension \\ 
    \hline
        coordinate descent & EN-SVM & ** & ** \\
        proximal gradient  & EN-SVM & *** & ** \\
        ADMM               & EN-SVM, SFL-SVM & **** & **** \\
        LAR                & EN-SVM & * & * \\
        quadratic program  & SFL-SVM & * & * \\
    \hline
    \end{tabular}
        \begin{tablenotes}
            \footnotesize
            \item[* indicates low, ** indicates medium, *** indicates good, and **** indicates excellent] 
        \end{tablenotes}
    \end{threeparttable}}
\end{table}
As shown in Table \ref{tab11}, ADMM-type algorithms have significant advantages in model generality and parallelization potential, making them an ideal foundation for constructing a unified solving framework. However, successfully applying them to the sparse SVM family that encompasses various non-convex regularization terms and establishing a solid convergence theory remains an open challenge. The work in this paper is precisely a deep extension and innovation based on the ADMM framework, aiming to provide a unified solution that can seamlessly cover the various models listed in the table, support efficient distributed computing, and have strict theoretical guarantees. 
\subsection{Consensus optimization framework based on ADMM}\label{sec2.2}
The ADMM has become an important tool for distributed optimization due to its natural adaptability to separable problems and excellent convergence properties. In the scenario of parallel data storage, the consensus problem form of ADMM proposed by Boyd et al. \cite{Boyd2010DistributedOA} provides an efficient solution paradigm.

Consider a typical distributed learning problem: assume that the data is distributed across $K$ computing nodes, each node holds a local dataset $\bm D_k, k=1,2, \cdots, K$. The global objective is to minimize the sum of the loss functions of all nodes:
\begin{equation}
    \min_{\beta_0,\bm \beta} \sum_{k = 1}^{K} \mathcal{\bm L}_k(\beta_0,\bm \beta)+\mathcal{\bm R}(\bm \beta)
\end{equation}
where $\mathcal{\bm L}_k$ is the loss function on node $k$ with data $\bm D_k$, and $\mathcal{\bm R}$ is the regularization term. By introducing local variable copies $\bm \beta_k$ and the global consensus variable $\bm z$, the original problem can be transformed into an equivalent consensus form:
\begin{equation}
    \label{comsensus}
    \begin{aligned}
    \min_{\{\bm \beta_k\}, \bm z}&\sum_{k = 1}^{K} \mathcal{\bm L}_k(\beta_0,\bm \beta_k)+\mathcal{\bm R}(\bm z)\\
    \text{s.t.}&\quad \bm \beta_k = \bm z, \ k = 1, 2, \cdots, K
    \end{aligned}
\end{equation}
where the regularization term $\mathcal{\bm R}(\bm z)$ is usually handled by the central node. 

The corresponding augmented Lagrangian function is:
\[
L_{\mu}(\beta_0,\{\bm \beta_k\}, \bm z, \{\bm b_k\}) = \sum_{k = 1}^{K} \mathcal{\bm L}_k(\beta_0,\bm \beta_k) + \mathcal{\bm R}(\bm z) + \bm b_k^{\top}(\bm \beta_k - \bm z) + \frac{\mu}{2} \|\bm \beta_k - \bm z\|^2
\]
where $\bm b_k$ represents the dual variable associated with the constraint $\bm \beta_k = \bm z$ for each local node $k$, and $\mu > 0$ is  the penalty parameter. 

The standard consensus ADMM updates iteratively according to the following steps:

$\bullet$ \textbf{Local node:} Each node $k$ solves the following problem in parallel:
\begin{align*}
    \bm \beta_k^{t + 1} & = \arg\min_{\bm \beta_k} \mathcal{\bm L}_k(\beta_0, \bm \beta_k) + (\bm b_k^t)^{\top} \bm \beta_k + \frac{\mu}{2} \|\bm \beta_k - \bm z^t\|^2,\\
    \bm b_k^{t + 1} & = \bm b_k^t + \mu (\bm \beta_k^{t + 1} - \bm z^{t}).
\end{align*}
$\bullet$ \textbf{Central node:} The central node aggregates the results:
\begin{align*}
    \beta_0^{t+1} & = \arg\min_{\beta_0} \sum_{k = 1}^{K} \mathcal{\bm L}_k(\beta_0,\bm \beta_k^{t+1}) \\
    \bm z^{t + 1} & = \arg\min_{\bm z} \mathcal{\bm R}(\bm z) - (\bm b_k^{t+1})^{\top} \bm z + \frac{\mu}{2} \|\bm \beta_k^{t + 1} - \bm z\|^2
\end{align*}
\subsection{Music genre classification method}\label{sec2.3}
The task of music genre classification typically adheres to a two-stage paradigm of feature representation and classifier design. Its research trajectory reflects the evolution of machine-learning techniques from those relying on expert knowledge to a data-driven paradigm.

Early research centered around classification methods based on handcrafted features. Drawing on domain expertise, researchers meticulously designed and extracted a series of statistical features from audio signals. These features included mel-frequency cepstral coefficients (MFCC), spectral centroid, and rhythm features, which were combined to form discriminative feature vectors. Subsequently, these feature vectors were fed into traditional classifiers such as SVM, Gaussian mixture models, or decision trees for training and prediction \cite{Lopes2010SelectionOT, Costa2011MusicHR}. Nevertheless, the quality of feature engineering is highly contingent on professional experience, and handcrafted features struggle to comprehensively capture the intricate non-linear structures and high-level semantics in music.

The advent of deep learning has spurred a shift in the research paradigm towards classification methods based on end-to-end learning. This approach takes the raw waveform of audio or time-frequency representations, such as Mel spectrograms, as input. Deep neural networks, primarily convolutional neural networks (CNN, \cite{Pelchat2020NeuralNM}) and recurrent neural networks (RNN, \cite{Yu2020DeepAB}), are employed to automatically learn the mapping from low-level signals to high-level genre labels. This approach bypasses the arduous process of handcrafted feature design. Fueled by large-scale data, it has achieved considerably higher classification accuracy than traditional methods on benchmark datasets like FMA. However, the performance improvement is accompanied by a deficit in interpretability. The black-box nature of deep models makes it challenging to discern the rationale behind classifying a particular piece of music into a specific genre. This limitation restricts its application in scenarios that demand human understanding and trust, such as music education and musicological analysis.

In recent years, the academic community has reevaluated the significance of interpretability, leading to two prominent trends. One trend is the development of post-hoc explanation tools, such as saliency maps, for high-performance deep models. The other is the exploration of models with inherent interpretability on novel data. This paper aligns with the latter approach. We will return to the framework of highly interpretable linear models, exemplified by SS-SVMs. By harnessing modern optimization theories and parallel computing technologies, we aim to equip these models to handle contemporary large-scale music data effectively. Our objective is to attain both competitive classification performance and clear and transparent feature-level explanations, thereby offering a quantitative and reliable computational tool for music style analysis.
\section{Unified optimization framework and parallel algorithm}\label{sec3}
In this section, we aim to construct a unified optimization framework and design its corresponding parallel distributed algorithm. 
\subsection{Unified optimization framework}\label{sec3.1}
Based on the previous review of various convex and non-convex structured sparse regularization terms, we propose the following unified optimization framework.
Let $\bm{1}_n$ represent $n$-dimensional vector whose elements are all 1, and $\bm{Y}$ be a diagonal matrix with its diagonal elements to be the vector $\bm{y}$. Take $\bar{\bm{X}} = \bm{YX}$, the loss functions can be express as 
\begin{equation}
\label{loss}
    \frac{1}{n}\sum_{i=1}^n \mathcal{L}[1-y_i(\bm x_i^T \bm \beta  + \beta_0)] = \mathcal{L}(\bm{1}_n - \bar{\bm{X}}\bm{\beta} - \bm{y}\beta_0)
\end{equation}
where $\mathcal{L}: \mathbb{R} \rightarrow [0, \infty)$ is a loss function. In the previously discussed SVM model, the loss function $\mathcal{L}$ is typically chosen as the hinge loss function. However, in the relevant research on SVM, there are also some other loss functions, which constitute important variants of the SVM model. Liang et al. \cite{Liang2024LinearizedAD} conducted a systematic introduction in this regard . They summarized six commonly used loss functions in the SVM model, namely hinge loss, least squares loss, huberized hinge loss, squared hinge loss, pinball loss, and huberized pinball loss, and provided the corresponding proximal operators for these loss functions. This work can be smoothly linked to the research in this paper.

Taking into account the extensions of various structured sparse regularization terms comprehensively, we obtain the \textbf{unified optimization form} of the SS-SVM model:
\begin{equation}
    \label{sr-SVM}
    \mathop {\min }\limits_{\beta_0,\bm{\beta}} \mathcal{L}(\bm{1}_n - \bar{\bm{X}}\bm{\beta} - \bm{y}\beta_0) + \mathcal{R}_{\lambda_1}(|\bm{\beta}|) + \mathcal{R}_{\lambda_2}(\bm{G\beta}).
\end{equation}
where $\bm{G}$ is a matrix that changes with the structured sparse regularization terms. For elastic net and sparse group Lasso, $\bm{G}$ is the $n$-dimensional identity matrix, and for sparse fused Lasso, $\bm{G}$ is $\bm{F}$.

This unified form is highly versatile. In terms of regularization terms, it covers three basic structures, namely the elastic net, sparse group Lasso, and sparse fused Lasso, as well as their corresponding non-convex (MCP/SCAD) variants, resulting in a total of 9 combinations. Regarding loss functions, thanks to the work of \cite{Liang2024LinearizedAD}, 6 common loss functions can be seamlessly integrated into this framework. Therefore, a total of $6 \times 9 = 54$ different SS-SVM models can be derived from this framework, which can flexibly adapt to different data characteristics and problem scenarios. Our subsequent algorithm design will be based on this unified framework, which means that the proposed algorithm also has a high degree of versatility and effectiveness. It can efficiently solve these 54 models, thus providing more flexible and accurate solutions for diverse practical problems and promoting the application of SVM models in a wider range of fields.
\subsection{Parallel algorithm design}\label{sec3.2} 
To efficiently solve the optimization problem under this unified framework (\ref{sr-SVM}), especially in the scenario where large-scale data is stored in parallel, we next propose a distributed parallel algorithm based on the consensus ADMM. 

Supported that the original dataset $\{\bm{X}, \bm{y}\} = \{(\bm{x}_i, y_i)\}_{i=1}^n$ is partitioned into $K$ non-overlapping blocks in the context of distributed storage and stored on $K$ machines. Each machine $k$ only has access to its local data $\{\bm X_k, \bm y_k\}, k=1,2, \cdots, K$ and hence local loss function $\mathcal{L}(\bm{1}_k - \bar{\bm X}_k\bm{\beta}_k - \bm{y}_k\beta_0)$ and the regularization $\mathcal{R}_{\lambda_1}(|\bm{\beta}|) + \mathcal{R}_{\lambda_2}(\bm{G\beta})$. Then the formula (\ref{sr-SVM}) can be rewritten as
\begin{equation}
    \label{distributed structured sparse SVM}
    \mathop {\min }\limits_{\{\bm \beta_k\}, \beta_0,\bm{\beta}} \sum_{k=1}^K \mathcal{L}(\bm{1}_k - \bar{\bm X}_k\bm{\beta}_k - \bm{y}_k\beta_0) + \mathcal{R}_{\lambda_1}(|\bm{\beta}|) + \mathcal{R}_{\lambda_2}(\bm{G\beta}).
\end{equation}
where $\bm \beta_k \in \mathbb{R}^p$, $\bar{\bm X}_k \in \mathbb{R}^{n_k \times p}$, $\bm 1_k \in \mathbb{R}^{n_k}, \bm y_k \in \mathbb{R}^{n_k}$, and $\sum_{k=1}^K n_k = n$.

To deal with the non-differentiability of the loss function and the non-separability of the structured sparse regularization terms, we introduced additional auxiliary variables $\bm \theta (\in \mathbb{R}^p$ for elastic net and sparse group Lasso; $\in \mathbb{R}^{p-1}$ for sparse fused Lasso), $\bm \xi_k \in \mathbb{R}^{n_k}, k = 1,2, \cdots, K$ and reformulate problem (\ref{distributed structured sparse SVM}) into an equivalent form:
\begin{equation}
    \label{distributed-problem}
    \begin{aligned}
        \mathop {\min }\limits_{\{\bm \beta_k\}, \bm \theta, \{\bm \xi_k\}, \beta_0,\bm{\beta}} & \sum_{k=1}^K \mathcal{L}(\bm \xi_k) + \mathcal{R}_{\lambda_1}(|\bm{\beta}|) + \mathcal{R}_{\lambda_2}(\bm{\theta}) \\
        \text{s.t.} \qquad \quad \ & \bm \xi_k = \bm{1}_k - \bar{\bm X}_k\bm{\beta}_k - \bm{y}_k\beta_0, \ k = 1,2, \cdots, K; \\
        & \bm \beta_k = \bm \beta, \ k = 1,2, \cdots, K; \\
        & \bm \theta = \bm{G\beta} 
    \end{aligned}
\end{equation}

The corresponding augmented Lagrangian function of (\ref{distributed-problem}) is
\begin{equation*}
    \label{lagrangian-function}
    \begin{aligned}
    L_{\mu} (\bm \beta_k, \bm \theta, \bm \xi_k, \beta_0, \bm \beta,  \bm b_k, \bm d_k, \bm e) & = \sum_{k=1}^K \mathcal{L}(\bm \xi_k) + \mathcal{R}_{\lambda_1}(|\bm{\beta}|) + \mathcal{R}_{\lambda_2}(\bm{\theta}) - \sum_{k=1}^K <\bm b_k, \bm \xi_k - \bm{1}_k + \bar{\bm X}_k\bm{\beta}_k + \bm{y}_k\beta_0> \\
    & - \sum_{k=1}^K <\bm d_k, \bm \beta_k - \bm \beta> - <\bm e, \bm \theta - \bm{G\beta}> + \frac{\mu}{2}\sum_{k=1}^K \|\bm \xi_k - \bm{1}_k + \bar{\bm X}_k\bm{\beta}_k + \bm{y}_k\beta_0\|_2^2 \\
    & + \frac{\mu}{2} \sum_{k=1}^K \|\bm \beta_k - \bm \beta\|_2^2 + \frac{\mu}{2} \|\bm \theta - \bm{G\beta}\|_2^2
    \end{aligned}
\end{equation*}
where $\bm b_k, \bm d_k, \bm e$ are dual variables corresponding to the linear constraints of (\ref{distributed-problem}), $\mu > 0$ is the penalty parameter, and $< \cdot >$ represents the standard inner product in Euclidean space. 

After giving the initial values $(\bm \beta_k^0, \bm \xi_k^0, \bm \theta^0, \beta_0^0, \bm \beta^0,  \bm b_k^0, \bm d_k^0, \bm e^0)$, the iterative process of the parallel ADMM can be expressed as
\begin{subequations}
    \label{iteration}
    \begin{align}
    \bm \beta_k^{t+1}  & \leftarrow \mathop {\arg\min }\limits_{\bm{\beta}_k} L_{\mu} (\bm \beta_k, \bm \theta^t, \bm \xi_k^t, \bm \beta^t, \beta_0^t, \bm b_k^t, \bm d_k^t, \bm e^t), \ \ k = 1,2, \cdots, K \label{iteration-betak} 
    \\
    \bm \xi_k^{t+1}  & \leftarrow \mathop {\arg\min }\limits_{\bm{\xi}_k} L_{\mu} (\bm \beta_k^{t+1}, \bm \theta^{t}, \bm \xi_k, \bm \beta^t, \beta_0^t, \bm b_k^t, \bm d_k^t, \bm e^t), \ \ k = 1,2, \cdots, K \label{iteration-xi} 
    \\
    \bm \theta^{t+1} & \leftarrow \mathop {\arg\min }\limits_{\beta_0} L_{\mu} (\bm \beta_k^{t+1}, \bm \theta, \bm \xi_{k+1}^t, \bm \beta^t, \beta_0^t, \bm b_k^t, \bm d_k^t, \bm e^t) \label{iteration-theta} 
    \\
    \bm \beta^{t+1} & \leftarrow \mathop {\arg\min }\limits_{\bm{\beta}} L_{\mu} (\bm \beta_k^{t+1}, \bm \theta^{t+1}, \bm \xi_k^{t+1}, \bm \beta, \beta_0^t, \bm b_k^t, \bm d_k^t, \bm e^t) \label{iteration-beta}
    \\
    \beta_0^{t+1}  & \leftarrow \mathop {\arg\min }\limits_{\beta_0} L_{\mu} (\bm \beta_k^{t+1}, \bm \theta^{t+1}, \bm \xi_k^{t+1}, \bm \beta^{t+1}, \beta_0, \bm b_k^t, \bm d_k^t, \bm e^t) \label{iteration-beta0} 
    \\
    \bm b_k^{t+1} & \leftarrow \bm b_k^t - \mu(\bm \xi_k^{t+1} - \bm{1}_k + \bar{\bm X}_k\bm{\beta}_k^{t+1} + \bm{y}_k\beta_0^{t+1}), \ \ k = 1,2, \cdots, K \label{iteration-b} 
    \\
    \bm d_k^{t+1} & \leftarrow \bm d_k^t - \mu(\bm \beta_k^{t+1} - \bm \beta^{t+1}), \ \ k = 1,2, \cdots, K \label{iteration-d} 
    \\
    \bm e^{t+1} & \leftarrow \bm e^t - \mu(\bm \theta^{t+1} - \bm{G\beta}^{t+1}) \label{iteration-e}
\end{align}
\end{subequations}
To save time, the iterative process for implementing $\bm \beta_k^{t + 1}$ in \eqref{iteration-betak}, $\bm \xi_k^{t + 1}$ in \eqref{iteration-xi}, $\bm b_k^{t + 1}$ in \eqref{iteration-b}, and $\bm d_k^{t + 1}$ in \eqref{iteration-d} is parallelized.  Next, we will provide detailed closed-form solutions for each subproblem in the iterative process.
\subsubsection{For the problem (\ref{iteration-betak})}\label{sec3.2.1}
\vspace{-1em}
\quad \ Regarding the update of $\bm \beta_k^{t+1}, k = 1,2, \cdots, K$, after rearranging the terms in (\ref{iteration-betak}) and omitting constant terms, it can be expresses as:
\begin{equation}
    \label{betak}
    \bm \beta_k^{t+1} \leftarrow \mathop {\arg\min }\limits_{\bm \beta_k} \left\{\frac{\mu}{2}\sum_{k=1}^K \left\|\bm \xi_k^{t} - \bm 1_k + \bar{\bm X}_k\bm \beta_k + \bm y_k \beta_0^{t} - \frac{\bm b_k^t}{\mu}\right\|_2^2 + \frac{\mu}{2}\sum_{k=1}^K\left\|\bm \beta_k - \bm \beta^{t} - \frac{\bm d_k^t}{\mu}\right\|_2^2\right\}.
\end{equation}
This problem is quadratic and differentiable, and we can directly obtain its result by taking the derivative.
\begin{equation}
    \label{betak-result}
    \bm \beta_k^{t+1} \leftarrow (\bar{\bm X}_k^T\bar{\bm X}_k + \bm I_p)^{-1} \left[ \bar{\bm X}_k^T( -\bm \xi_k^{t} + \bm 1_k - \bm y_k \beta_0^{t} + \frac{\bm b_k^t}{\mu}) + \bm \beta^{t} + \frac{\bm d_k^t}{\mu}\right].
\end{equation}
In the process of solving $\bm \beta_k$, the calculation of the matrix inverse is a crucial step, and the complexity of this calculation is closely related to the number of rows $n_k$ and the number of columns $p$ of the matrix $\bar{\bm X}_k$. When $p > n_k$, Yu et al. \cite{Yu2017APA} proposed the Woodbury matrix identity can be utilized, that is, $(\bar{\bm X}_k^T\bar{\bm X}_k + \bm I_p)^{-1}=\bm I_p - \bar{\bm X}_k^T (\bm I_{n_k} + \bar{\bm X}_k \bar{\bm X}_k^T)^{-1}\bar{\bm X}_k$. This method has obvious advantages when $n_k$ is small because the dimension of the matrix $\bm I_{n_k}+\bar{\bm X}_k\bar{\bm X}_k^T$ is small, and its inverse only needs to be calculated once during the ADMM iteration. When both $n_k$ and $p$ are large, directly calculating the matrix inverse is time-consuming and may cause numerical stability issues. In this case, the conjugate gradient method introduced in \cite{Saad2003IterativeMF} is more appropriate. It is an efficient iterative algorithm for solving symmetric positive-definite linear equations and has significant advantages in dealing with large-scale sparse problems. 
\subsubsection{For the problem (\ref{iteration-xi})}\label{sec3.2.3}
\vspace{-1em}
\quad Regarding the update of $\bm \xi_k^{t+1}, k = 1,2, \cdots, K$, after rearranging the terms in (\ref{iteration-xi}) and omitting constant terms, it can be expresses as:
\begin{equation}
    \label{xi}
    \bm \xi_k^{t+1} \leftarrow \mathop {\arg\min }\limits_{\bm \xi_k} \left\{\sum_{k=1}^K \left[\mathcal{L}(\bm \xi_k) + \frac{\mu}{2} \|\bm \xi_k - \bm 1_k + \bar{\bm X}_k \bm \beta_k^{t+1} + \bm y_k \beta_0^{t} - \frac{\bm b_{k_i}^t}{\mu}\|_2^2\right]\right\}.
\end{equation}
where $\mathcal{L}(\bm \xi_k) = \frac{1}{n_k} \sum_{i=1}^{n_k} L(\xi_{k_i})$ is separable. Then the problem (\ref{xi}) can be expressed in the form of components.
\begin{equation}
    \label{xi-component}
    \xi_{k_i}^{t+1} \leftarrow \mathop {\arg\min }\limits_{\xi_{k_i}} \left\{\sum_{i=1}^{n_k} \left[\frac{1}{n_k}L( \xi_{k_i}) + \frac{\mu}{2}\left(\xi_{k_i} - 1 + y_{k_i}(\bm x_{k_i}^T \bm \beta_k^{t+1} + \beta_0^{t}) - \frac{b_i^t}{\mu}\right)^2\right]\right\}.
\end{equation}
Obviously, the problem (\ref{xi-component}) can be separated into $n$ proximal operators, which is defined as $\text{prox}_{s, {L}}(\zeta) = \mathop {\arg\min }\limits_{x} \{\frac{1}{n}L(x) + s/2(x - \zeta)^2\}$. And its closed-form solution will vary with the change of ${L}$, that is, the loss function.

In this paper, we consider six loss functions, namely least squares loss, hinge loss, huberized hinge loss, squared hinge loss, pinball loss and huberized pinball loss. The closed-form solutions of the proximal operators for these six loss functions have been fully discussed in Table 1 of \cite{Liang2024LinearizedAD}. We will not elaborate on them here but directly present the closed-form solution of equation (\ref{xi}) in Table \ref{tab4} in the form of components. For the convenience of expression, we introduce the symbol $\zeta_{k_i}^{t+1} = 1 - y_{k_i}(\bm x_{k_i}^T \bm \beta_k^{t+1} + \beta_0^{t}) + \frac{b_{k_i}^t}{\mu}$.
\begin{table}[H]
    \centering
    \caption{The closed-form solutions of (\ref{xi-component}) under different loss functions}
    \renewcommand{\arraystretch}{1.5}
    \setlength{\belowcaptionskip}{0.2cm} 
    \setlength{\tabcolsep}{7mm}{
    \begin{tabular}{ll}
    \hline
    Loss function  & Closed-form solutions of (\ref{xi-component}) \\
    \hline
     Hinge             &  $\xi_{k_i}^{t+1} = \max\{\zeta_{k_i}^{t+1} - \frac{1}{n_k\mu}, \min(0,\zeta_{k_i}^{t+1})\}$ \\
     Least square      &  $\xi_{k_i}^{t+1} = \begin{cases}
         \frac{n_k\mu\zeta_{k_i}^{t+1}}{n_k\mu + 2\zeta_{k_i}^{t+1}}, & \zeta_{k_i}^{t+1} \ge 0;\\
         \frac{n_k\mu\zeta_{k_i}^{t+1}}{n_k\mu + 2(1-\tau)}, & \zeta_{k_i}^{t+1} < 0.
     \end{cases}$ \\
     Square hinge      &  $\xi_{k_i}^{t+1} = \begin{cases}
         \frac{n_k\mu\zeta_{k_i}^{t+1}}{n_k\mu+1}, & \zeta_{k_i}^{t+1} \ge 0; \\
         \zeta_{k_i}^{t+1}, & \zeta_{k_i}^{t+1} < 0.
     \end{cases}$ \\
     Huberized hinge   &  $\xi_{k_i}^{t+1} = \max\{\zeta_{k_i}^{t+1} - \frac{1}{n_k\mu}, \min(\zeta_{k_i}^{t+1}, \frac{n_k\delta \mu \zeta_{k_i}^{t+1}}{1 + n_k\delta \mu})\}$ \\
     Pinball           &  $\xi_{k_i}^{t+1} = \max\{\zeta_{k_i}^{t+1} - \frac{1}{n_k\mu}, \min(0,\zeta_{k_i}^{t+1} + \frac{\tau}{n_k\mu})\}$ \\
     Huberized pinball &  $\xi_{k_i}^{t+1} = \begin{cases}
         \frac{n_k\mu\delta\zeta_{k_i}^{t+1}}{n_k\mu\delta + 1} + \frac{1}{n_k\mu}\left(\frac{n_k\mu\delta\zeta_{k_i}^{t+1}}{n_k\mu\delta + 1} - 1 \right) I(\zeta_{k_i}^{t+1} > \frac{1}{n_k\mu} + \delta), & \zeta_{k_i}^{t+1} > 0;\\
         \frac{n_k\mu\delta\zeta_{k_i}^{t+1}}{n_k\mu\delta + \tau} + \frac{\tau}{n_k\mu} \left( \frac{n_k\mu\delta\zeta_{k_i}^{t+1}}{n_k\mu\delta + \tau} + 1 \right) I(\zeta_{k_i}^{t+1} \leq -\frac{\tau}{n_k\mu} - \delta), & \zeta_{k_i}^{t+1} \leq 0.
     \end{cases}$ \\
     \hline
    \end{tabular}}
    \label{tab4}
\end{table}
\subsubsection{For the problem (\ref{iteration-theta})}\label{sec3.2.2}
\vspace{-1em}
\quad \ Regarding the update of $\bm \theta^{t+1}$, after rearranging the terms in (\ref{iteration-theta}) and omitting constant terms, it can be expresses as:
\begin{equation}
    \label{theta}
    \bm \theta^{t+1} \leftarrow \mathop {\arg\min }\limits_{\bm \theta} \left\{\mathcal{R}_{\lambda_2}(\bm \theta) + \frac{\mu}{2} \|\bm \theta - \bm{G\beta}^{t} - \frac{\bm e^t}{\mu} \|_2^2\right\}.
\end{equation}
The result of (\ref{theta}) depends on $\mathcal{R}_{\lambda_2}(\bm \theta)$ and $\bm G$, both of which are determined by structured sparse regularization terms. Therefore, we will discuss the specific updated results of $\bm \theta$ according to different structured sparse regularization terms.
\\
\\
$\bullet$ For the EN regularization, $\mathcal{R}_{\lambda_2}(\bm \theta) = \lambda_2 \|\bm \theta\|_2^2$ and $\bm G = \bm I_n$.  Formula (\ref{theta}) can be express as
\begin{equation}
    \label{theta-en}
    \bm \theta^{t+1} \leftarrow \mathop {\arg\min }\limits_{\bm \theta} \left\{\lambda_2 \|\bm \theta\|_2^2 + \frac{\mu}{2} \|\bm \theta - \bm{\beta}^{t} - \frac{\bm e^t}{\mu} \|_2^2\right\}.
\end{equation}
The minimization problem is quadratic and differentiable, allowing us to solve
the subproblem by solving the following linear equations,
\begin{equation}
    \label{theta-result-1}
    \bm \theta^{t+1} \leftarrow \frac{\mu(\bm \beta^{t} + \bm e^t/\mu)}{2\lambda_2 + \mu}.
\end{equation}
\\
\\
$\bullet$ For the SFL regularization, $\mathcal{R}_{\lambda_2}(\bm \theta) = \lambda_2 \|\bm \theta\|_1$ and $\bm G = \bm F$.  Formula (\ref{theta}) can be express as
\begin{equation}
    \label{theta-sfl}
    \bm \theta^{t+1} \leftarrow \mathop {\arg\min }\limits_{\bm \theta} \left\{\lambda_2 \|\bm \theta\|_1 + \frac{\mu}{2} \|\bm \theta - \bm{F\beta}^{t} - \frac{\bm e^t}{\mu} \|_2^2\right\}.
\end{equation}
This is a soft-thresholding operator, and we can directly present the result. 
\begin{equation}
    \label{theta-result-2}
    \bm \theta^{t+1} \leftarrow \text{Shrink} \left[ \bm{F\beta}^{t} + \frac{\bm e^t}{\mu}, \frac{\lambda_2}{\mu} \right],
\end{equation}
where $\text{Shrink}[a,b] = \text{sign}(a) \max\{|a| - b\}$.
\\
\\
$\bullet$ For the SGL regularization, $\mathcal{R}_{\lambda_2}(\bm \theta) = \lambda_2 \|\bm \theta\|_{2,1}$ and $\bm G = \bm I_n$. Formula (\ref{theta}) can be express as
\begin{equation}
    \label{theta-sgl}
    \bm \theta^{t+1} \leftarrow \mathop {\arg\min }\limits_{\bm \theta} \left\{\lambda_2 \|\bm \theta\|_{2,1} + \frac{\mu}{2} \|\bm \theta - \bm{\beta}^{t} - \frac{\bm e^t}{\mu} \|_2^2\right\}.
\end{equation}
By recalling the definition of the $\ell_{2,1}$ norm $(\|\bm \theta\|_{2,1} = \sum_{m=1}^M\|\bm{\theta}_{(m)}\|_2)$, we can find that $\bm \theta$ is divided into $M$ groups. Since the whole equation (\ref{theta-sgl}) is separable, the above formula can be regarded as the sum of $M$ $\ell_2$ norm proximal operators, which is defined as $\text{prox}_{s,\lambda\|\bm \theta\|_{2,1}}(\bm x) = \mathop {\arg\min }\limits_{\bm \theta} \{\lambda \|\bm \theta\|_{2,1} + s/2\|\bm \theta - \bm x\|_2^2\}$. The closed-form solution for $\text{prox}_{s,\lambda\|\bm \theta\|_{2,1}}$ was first presented by \cite{Yuan2006ModelSA} and can be expressed as 
\[\text{prox}_{s,\lambda\|\bm \theta\|_{2,1}}(\bm x) = \frac{\bm x}{\|\bm x\|_2} \cdot \left[\|\bm x\|_2 - \frac{\lambda}{s}\right]_+. \]
Then, we have
\begin{equation}
    \label{theta-result-3}
    \bm \theta_{(m)}^{t+1} \leftarrow \frac{\bm{\beta}_{(m)}^{t} + {\bm e_{(m)}^t}/{\mu}}{\|\bm{\beta}_{(m)}^{t} + {\bm e_{(m)}^t} /{\mu}\|_2} \cdot \left[\|\bm{\beta}_{(m)}^{t} + \frac{\bm e_{(m)}^t}{\mu}\|_2  - \frac{\lambda_2}{\mu}\right]_+, m = 1,2, \cdots, M.
\end{equation}
For convenience, we summarize the updates of $\bm \theta$ under different regularization terms in Table \ref{tab3}.
\begin{table}[H]
    \centering
    \caption{The updated results of $\bm \theta^{t+1}$ under the different structured sparse regularization terms}
    \renewcommand{\arraystretch}{1.5}
    \setlength{\belowcaptionskip}{0.3cm} 
    \setlength{\tabcolsep}{7mm}{
    \begin{tabular}{ll}
    \hline
    CRT  &  The updated result of $\bm \theta^{t+1}$\\
    \hline
    EN   &  $\bm \theta^{t+1} \leftarrow \frac{\mu(\bm \beta^{t} + \bm e^t/\mu)}{2\lambda_2 + \mu}$ \\
    SFL  &  $\bm \theta^{t+1} \leftarrow \text{Shrink} \left[ \bm{F\beta}^{t} + \frac{\bm e^t}{\mu}, \frac{\lambda_2}{\mu} \right]$ \\
    SGL  &  \makecell[c]{$\bm \theta_{(m)}^{t+1} \leftarrow \frac{\bm{\beta}_{(m)}^{t} + {\bm e_{(m)}^t}/{\mu}}{\|\bm{\beta}_{(m)}^{t} + {\bm e_{(m)}^t}/{\mu}\|_2} \cdot \left[\|\bm{\beta}_{(m)}^{t} + \frac{\bm e_{(m)}^t}{\mu}\|_2  - \frac{\lambda_2}{\mu}\right]_+, m = 1,2, \cdots, M$} \\
    \hline
    \end{tabular}}
    \label{tab3}
\end{table}
\subsubsection{For the problem (\ref{iteration-beta})}\label{sec3.2.4}
\vspace{-1em}
\quad \ Regarding the update of $\bm \beta^{t+1}$, after rearranging the terms in (\ref{iteration-beta}) and omitting constant terms, it can be expresses as:
\begin{equation}
    \label{beta}
    \bm \beta^{t+1} \leftarrow \mathop {\arg\min }\limits_{\bm{\beta}} \left\{ P_{\lambda_1}(|\bm \beta|) + \frac{\mu}{2}\sum_{k=1}^K\|\bm \beta_k^{t+1} - \bm \beta - \frac{\bm d_k^t}{\mu}\|_2^2 + \frac{\mu}{2}\|\bm \theta^{t+1} - \bm{G\beta} - \frac{\bm e^t}{\mu}\|_2^2 \right\}
\end{equation}
Evidently, the specific result of the above formula hinges on the selection of $P_{\lambda_1}(|\bm \beta|)$ and $\bm G$. Next, we will conduct a case-by-case discussion. 
\\
\\
$\bullet$ When $P_{\lambda_1}(|\bm \beta|) = \lambda_1 \|\bm \beta\|_1$ and $\bm G = \bm I_n$, corresponding to the EN and SGL regularization terms, we rearrange the optimization equation and eliminate certain constant terms that have no bearing on the optimization target variable $\bm \beta$, then the subsequent equation can be derived.
\begin{equation}
    \label{convex-I-beta}
    \bm \beta^{t+1} \leftarrow \left\{ \mathop {\arg\min }\limits_{\bm{\beta}} \lambda_1\|\bm \beta\|_1 + \frac{\mu(K+1)}{2} \|\bm \beta - \frac{K (\bar{\bm \beta}^{t+1} - \bar{\bm d}^t/\mu) + \bm \theta^{t+1} - \bm e^t/\mu}{K+1}\|_2^2 \right\}
\end{equation}
where $\bar{\bm \beta}^{t+1} = K^{-1}\sum_{k=1}^K \bm \beta_k^{t+1}$ and $\bar{\bm d}^t = K^{-1} \sum_{k=1}^K \bm d_k^t$. Evidently, the above optimization problem (\ref{convex-I-beta}) corresponds to a soft-thresholding operator, and we can directly obtain the result.
\begin{equation}
    \label{beta-result-1}
    \bm \beta^{t+1} \leftarrow \text{Shrink} \left[ \frac{K (\bar{\bm \beta}^{t+1} - \bar{\bm d}^t/\mu) + \bm \theta^{t+1} - \bm e^t/\mu}{K+1}, \frac{\lambda_1}{\mu{K+1}} \right]
\end{equation}
\\
\\
$\bullet$ When $P_{\lambda_1}(|\bm \beta|) = \lambda_1 \|\bm \beta\|_1$ and $\bm G = \bm F$, corresponding to the SFL regularization term, we can get
\begin{equation}
    \label{convex-F-beta}
    \bm \beta^{t+1} \leftarrow \left\{ \mathop {\arg\min }\limits_{\bm{\beta}} \lambda_1\|\bm \beta\|_1 + \frac{\mu}{2} \bm \beta^T (K \bm I_p + \bm F^T \bm F) \bm \beta - \mu \bm \beta^T \left[K(\bar{\bm \beta}^{t+1} - \frac{\bar{\bm d}^t}{\mu}) + \bm F^T (\bm \theta^{t+1} - \frac{\bm e^t}{\mu})\right] \right\}
\end{equation}
It is quite evident that, owing to the non-identity matrix preceding the quadratic term of $\bm \beta$ and the presence of $\|\bm \beta\|_1$, Equation (\ref{convex-F-beta}) does not possess a closed-form solution. While numerical approaches like the coordinate descent method can be employed to solve  (\ref{convex-F-beta}), this would substantially elevate the computational load. In this context, we put forward the use of a linearization technique to approximate this optimization problem and acquire a closed-form solution for the problem (\ref{convex-F-beta}). 

We can linearize the quadratic term $\frac{\mu}{2} \bm \beta^T (K \bm I_p + \bm F^T \bm F) \bm \beta - \mu \bm \beta^T \left[K(\bar{\bm \beta}^{t+1} - \frac{\bar{\bm d}^t}{\mu}) + \bm F^T (\bm \theta^{t+1} - \frac{\bm e^t}{\mu})\right]$ in (\ref{convex-F-beta}) and replace it by
\begin{equation}
    \label{linearization}
    \left\{\mu (K\bm I_p + \bm F^T \bm F)^T \bm {\beta^t} - \mu \left[K(\bar{\bm \beta}^{t+1} - \frac{\bar{\bm d}^t}{\mu}) + \bm F^T (\bm \theta^{t+1} - \frac{\bm e^t}{\mu})\right] \right\} (\bm \beta - \bm \beta^t) + \frac{\eta}{2}\|\bm \beta - \bm \beta^t\|_2^2
\end{equation}
where the linearization parameter $\eta > 0$ serves to regulate the closeness between $\bm \beta$ and $\bm \beta^t$. To guarantee the convergence of the algorithm, $\eta$ must exceed the largest eigenvalue of the matrix $(K \bm I_p + \bm F^\top \bm F)$. It is worth noting that due to the special structure of the matrix $\bm F$, the largest eigenvalue of $(K \bm I_p + \bm F^\top \bm F)$ is exactly $K + 4$. As a result, we can conveniently set $\eta = K + 4.01$. Consequently, we can tackle the following problem
\begin{equation}\small
    \label{beta-linearization}
    \bm \beta^{t+1} \leftarrow \mathop {\arg\min }\limits_{\bm{\beta}} \left\{ \lambda_1\|\bm \beta\|_1 + \frac{\eta}{2} \left\|\bm \beta - \bm \beta^t + \frac{\mu}{\eta} \left[ (K\bm I_p + \bm F^T \bm F) \bm {\beta^t} - K(\bar{\bm \beta}^{t+1} - \frac{\bar{\bm d}^t}{\mu}) - \bm F^T (\bm \theta^{t+1} - \frac{\bm e^t}{\mu}) \right]\right\|_2^2 \right\}
\end{equation}
to attain an approximate solution for the (\ref{convex-F-beta}).
\begin{equation}
    \label{beta-result-2}
    \bm \beta^{t+1} \leftarrow \text{Shrink} \left[ \bm \beta^t - \frac{\mu}{\eta}\left( (K\bm I_p + \bm F^T \bm F) \bm {\beta^t} - K(\bar{\bm \beta}^{t+1} - \frac{\bar{\bm d}^t}{\mu}) - \bm F^T (\bm \theta^{t+1} - \frac{\bm e^t}{\mu}) \right) ,\frac{\lambda_1}{\eta}\right]
\end{equation}
\\
\\
$\bullet$ When $P_{\lambda_1}(|\bm \beta|) = \text{SCAD or MCP}$, corresponding to the non-convex variant of SS-SVM, we need to introduce the local linear approximation (LLA) method proposed by \cite{Zou2008OneSE}. As stated in \cite{Zou2008OneSE}, non-convex regularization terms can be approximated as
\begin{equation}
    \label{lla}
    P_{\lambda}(|\beta_j|) \approx P_{\lambda}(|\beta_j^{(0)}|) + P_{\lambda}^{'} (|\beta_j^{(0)}|) (|\beta_j| - |\beta_j^{(0)}|), \ \text{for} \ \beta_j \approx \beta_j^{(0)}, \ j = 1,2, \cdots, p,
\end{equation}
where the relevant results of $P_{\lambda}^{'} (|\beta_j|)$ are summarized in the Table \ref{tab5}. In addition, it should be emphasized that the accuracy of the LLA algorithm is affected by the selection of the initial value $\beta_j^{(0)}$. In this paper, we adopt the method proposed in \cite{Gu2018ADMMFH}, that is, we use the solution of the $\ell_1$-SVM as the initial value $\beta_j^{(0)}$.
\begin{table}[H]
    \centering
    \renewcommand{\arraystretch}{1.5}
    \setlength{\belowcaptionskip}{0.2cm} 
    \caption{The $P_{\lambda}^{'} (|\beta_j|)$ of SCAD and MCP}
    \setlength{\tabcolsep}{7mm}{
    \begin{tabular}{ll}
    \hline
    Name & \qquad \qquad \qquad Formula \\ \hline
     SCAD $(a > 2)$ & $P_\lambda^{'}(|\beta_j|) = \begin{cases}
    \lambda_1, & \text{{if }} |\beta_j| \leq \lambda_1, \\[3mm]
    \frac{a\lambda_1 -  |\beta_j|}{a-1}, & \text{{if }} \lambda_1 < |\beta_j|  < a  \lambda_1, \\[3mm]
    0 , & \text{{if }} |\beta_j| \ge a  \lambda_1. \end{cases}$\\
    MCP $(a > 0)$ &  $P_\lambda^{'}(|\beta_j|) = \begin{cases}
    \lambda_1  - \frac{|\beta_j|}{a}, & \text{{if }}  |\beta_j|  \le a \lambda_1\\
    0, & \text{{if }}  |\beta_j|  > a \lambda_1 \\
    \end{cases}$\\
    \hline
    \end{tabular}}
    \label{tab5}
\end{table}

Then, the problem (\ref{beta}) can be approximately formulated as
\begin{equation}
    \label{lla-beta}
    \bm \beta^{t+1} \leftarrow \mathop {\arg\min }\limits_{\bm{\beta}} \left\{ \sum_{j=1}^p P_{\lambda_1}^{'}(|\beta_j^t|) |\beta_j| + \frac{\mu}{2}\sum_{k=1}^K\|\bm \beta_k^{t+1} - \bm \beta - \frac{\bm d_k^t}{\mu}\|_2^2 + \frac{\mu}{2}\|\bm \theta^{t+1} - \bm{G\beta} - \frac{\bm e^t}{\mu}\|_2^2 \right\}.
\end{equation}
We can observe that, compared with the problem (\ref{beta}), the only difference is that the parameter $\lambda$ in it has been replaced with $P_{\lambda_1}^{'}(|\beta_j^t|)$. This implies that we only need to make minor adjustments to the above results to to obtain the result of (\ref{lla-beta}). For $\ j = 1,2, \cdots, p,$
\begin{align}
    \label{lla-beta-result-1}
    & \beta_j^{t+1} \leftarrow \text{Shrink} \left[ \frac{K (\bar{\bm \beta}^{t+1} - \bar{\bm d}^t/\mu) + \bm \theta^{t+1} - \bm e^t/\mu}{K+1}, \frac{P_{\lambda_1}^{'}(|\beta_j^t|)}{\mu{K+1}} \right], \\
     \label{lla-beta-result-2}
    & \beta_j^{t+1} \leftarrow \text{Shrink} \left[ \bm \beta^t - \frac{\mu}{\eta}\left( (K\bm I_p + \bm F^T \bm F) \bm {\beta^t} - K(\bar{\bm \beta}^{t+1} - \frac{\bar{\bm d}^t}{\mu}) - \bm F^T (\bm \theta^{t+1} - \frac{\bm e^t}{\mu}) \right) ,\frac{P_{\lambda_1}^{'}(|\beta_j^t|)}{\eta}\right],  
\end{align}    
where the value of the parameter $P_{\lambda_1}^{'}(|\beta_j^t|)$ is calculated according to Table \ref{tab5}. The solutions in \eqref{lla-beta-result-1} and \eqref{lla-beta-result-2} are for $\bm G = \bm I_n$ and $\bm G = \bm F$ respectively. 

\subsubsection{For the problem (\ref{iteration-beta0})}\label{sec3.2.5}
\vspace{-1em}
\quad Regarding the update of $\bm \beta_0^{t+1}$, after rearranging the terms in (\ref{iteration-beta0}) and omitting constant terms, it can be expresses as:
\begin{equation}
    \label{beta0}
    \beta_0^{t+1} \leftarrow \mathop {\arg\min }\limits_{\beta_0} \left\{\frac{\mu}{2} \sum_{k=1}^K \|\xi_k^{t+1} - \bm{1}_k + \bar{\bm{X}_k}\bm{\beta}_k^{t+1} + \bm{y}_k\beta_0 - \frac{\bm b_k^t}{\mu} \|_2^2\right\}.
\end{equation}
This problem is quadratic and differentiable, and we can directly obtain its result,
\begin{equation}
    \label{beta0-result}
    \beta_0^{t+1} \leftarrow\sum_{k=1}^K \frac{1}{n_k} \bm y_k^T \left( -\xi_k^{t+1} + \bm{1}_k - \bar{\bm{X}_k}\bm{\beta}_k^{t+1} + \frac{\bm b_k^t}{\mu} \right).
\end{equation}
Since the update of $\beta_0^{t+1}$ requires the data on all local machines, to reduce the communication cost, we calculate 
\begin{equation}
    \label{beta0k}
    c_k^{t+1} = \frac{1}{n_k} \bm y_k^T \left( -\xi_k^{t+1} + \bm{1}_k - \bar{\bm{X}_k} \bm{\beta}_k ^{t+1} + \frac{\bm b_k^t}{\mu} \right), k = 1,2, \cdots, K,
\end{equation}
on each local machine respectively and then aggregate them to the central machine for addition.
\subsection{Gaussian Back Substitution}\label{sec3.3}
\vspace{-1em}
\quad \ Although we presented the solution to formula (\ref{distributed-problem}) in Section \ref{sec3.2}, the convergence of this solution cannot be guaranteed. The formula (\ref{distributed-problem}) involves $2K + 3$ primal variables, i.e., $\{\bm \beta_k, \bm \theta, \bm \xi_k, \bm \beta, {\beta_0}\}, k = 1,2, \cdots, K$, and the constraint matrices associated with these variables cannot be split into two mutually orthogonal groups. Therefore, it cannot be transformed into the traditional two-block ADMM algorithm to ensure convergence \cite{Chen2016TheDE}. To address this issue, we will employ the Gaussian back substitution method to correct certain iterative solutions in (\ref{iteration}). This method was proposed by \cite{He2012AlternatingDM} and has been widely used to ensure the convergence of multi-block ADMM algorithms,including but not limited to the works of \cite{Wu2025ParallelAA, Fu2019BlockAD, He2018ACO}.

For the convenience of description, we rewrite the formula (\ref{distributed-problem}) into a three-block optimization form. Let ${\bm v_1} = ({\bm \beta_1}^\top, {\bm \beta_2}^\top, \cdots, {\bm \beta_K}^\top, {\bm \theta}^\top)^\top,{\bm v_2} = ({\bm \xi_1}^\top, {\bm \xi_2}^\top, \cdots, {\bm \xi_K}^\top)^\top, {\bm v_3} = ({\bm \beta}^\top, {\beta_0})^\top$ and $\varphi_1({\bm v_1}) = \mathcal{R}_{\lambda_2}({\bm \theta})$, $\varphi_2({\bm v_2}) = \sum_{k=1}^K \mathcal{L}({\bm \xi_k}), \varphi_3({\bm v_3}) = P_{\lambda_1}(|{\bm \beta}|)$, the optimization problem (\ref{distributed-problem}) can be transformed as 
\begin{equation}
    \label{three-block}
    \begin{aligned}
    \mathop {\min }\limits_{{\bm v_1}, {\bm v_2}, {\bm v_3}} & \varphi_1 ({\bm v_1}) + \varphi_2 ({\bm v_2}) + \varphi_3 ({\bm v_3})\\
    \text{s.t.} \ \  & \bm{A}{\bm v_1} + \bm{B}{\bm v_2} + \bm{C}{\bm v_3} = \bm z,
    \end{aligned}
\end{equation}
where $\bm z = (\bm 1_1^\top, \bm 1_2^\top, \cdots, \bm 1_K^\top, \bm 0^\top, \bm 0^\top, \cdots, \bm 0^\top, \bm 0^\top)^\top$, the three matrices $\bm A, \bm B$, and $\bm C$ are defined as follows: 
\[\bm A = [\bm A_1, \bm A_2, \cdots, \bm A_K, \bm A_{K+1}] = \begin{bmatrix}
    \bar{\bm X_1}^\top & \bm 0^\top & \cdots & \bm 0^\top & \bm I_p & \bm 0^\top & \cdots & \bm 0^\top & \bm 0^\top \\
    \bm 0^\top & \bar{\bm X_2}^\top & \cdots & \bm 0^\top & \bm 0^\top & \bm I_p & \cdots & \bm 0^\top & \bm 0^\top \\
    \vdots & \vdots & \ddots & \vdots & \vdots & \vdots & \ddots & \vdots & \vdots \\
    \bm 0^\top & \bm 0^\top & \cdots & \bar{\bm X_K}^\top & \bm 0^\top & \bm 0^\top & \cdots & \bm I_p & \bm 0^\top \\
    \bm 0^\top & \bm 0^\top & \cdots & \bm 0^\top & \bm 0^\top & \bm 0^\top & \cdots & \bm 0^\top & \bm I_G
\end{bmatrix}^\top_{(K+1)\times(2K+1)}\]
\[\bm B = [\bm B_1, \bm B_2, \cdots, \bm B_K] = \begin{bmatrix}
    \bm I_{n_1} & \bm 0^\top & \cdots & \bm 0^\top & \bm 0^\top & \bm 0^\top & \cdots & \bm 0^\top & \bm 0^\top \\
    \bm 0^\top & \bm I_{n_2} & \cdots & \bm 0^\top & \bm 0^\top & \bm 0^\top & \cdots & \bm 0^\top & \bm 0^\top \\
    \vdots & \vdots & \ddots & \vdots & \vdots & \vdots & \ddots & \vdots & \vdots \\
    \bm 0^\top & \bm 0^\top & \cdots & \bm I_{n_K} & \bm 0^\top & \bm 0^\top & \cdots & \bm 0^\top & \bm 0^\top 
\end{bmatrix}^\top_{K\times(2K+1)}\]
\[\bm C = [\bm C_1, \bm C_2] = \begin{bmatrix}
    \bm 0^\top & \bm 0^\top & \cdots & \bm 0^\top & -\bm I_p & -\bm I_p & \cdots & -\bm I_p & -\bm G^\top \\
    \bm y_1^\top & \bm y_2^\top & \cdots & \bm y_K^\top & \bm 0^\top & \bm 0^\top & \cdots & \bm 0^\top & \bm 0^\top
\end{bmatrix}^\top_{2\times(2K+1)}.\]

The fundamental concept of the Gaussian back substitution method is that in each iteration, variables are initially updated through forward prediction to yield a set of approximate solutions. Subsequently, in a reverse back-substitution sequence, the approximate solutions are systematically corrected by leveraging the most recently updated variable values within the same iteration step. To distinguish them, we denote the approximate solutions as $(\tilde{\bm v}_2, \tilde{\bm v}_3)$, generated by (\ref{iteration}) , and the corrected solutions as $(\bm v_2, \bm v_3)$. Mathematically, the Gaussian back-substitution method can be characterized as a structured linear back substitution procedure described below.
\begin{equation}
    \label{gaussian}
    \begin{bmatrix}
        \bm v_2^{t+1} \\
        \bm v_3^{t+1}
    \end{bmatrix} = \begin{bmatrix}
        \bm v_2^{t} \\
        \bm v_3^{t}
    \end{bmatrix} - \begin{bmatrix}
        \nu \bm I & -\nu(\bm B^\top \bm B)^{-1} \bm B^\top \bm C \\
        \bm 0 & \nu \bm I
    \end{bmatrix} \begin{bmatrix}
        \bm v_2^t - \tilde{\bm v}_2^t \\
        \bm v_3^t - \tilde{\bm v}_3^t
    \end{bmatrix}
\end{equation}
where $\nu \in (0,1)$. From the definition of the constraint matrix $\bm{B}$, we know that $\bm{B}^{\top}\bm{B}$ is an identity matrix. This effectively avoids calculating the inverse of a large-scale matrix and significantly reduces the computational burden. And we have
\[
(\bm B^\top \bm B)^{-1}\bm B^\top \bm C = \bm B^\top \bm C = \begin{bmatrix}
    \bm 0 & \bm y_1 \\
    \bm 0 & \bm y_2 \\
    \vdots & \vdots \\
    \bm 0 & \bm y_k 
\end{bmatrix}.
\]

Then according to (\ref{gaussian}), it can be derived 
\begin{subequations}
    \label{gaussian-total}
    \begin{align}
    \bm \xi_k^{t+1} & = (1 - \nu)\xi_k^t + \nu \tilde{\bm \xi}_k^t + \nu \bm y_k (\beta_0^t - \tilde{\beta}_0^t), k = 1,2, \cdots, K; \label{gaussian-xi}\\
    \bm \beta^{t+1} & = (1 - \nu)\bm \beta^t + \nu \tilde{\bm \beta}^t, k = 1,2, \cdots, K; \label{gaussian-beta} \\
    \beta_0^{t+1} & = (1 - \nu)\beta_0^t + \nu \tilde{\beta}_0^t. \label{gaussian-beta0}
    \end{align}
\end{subequations}
where $\tilde{\bm \xi}_k^t, \tilde{\bm \beta}^t$, and $\tilde{\beta}_0^t$ are generated by (\ref{iteration}). Since there is no need to make corrections to $\tilde{\bm v}_1^{t+1}$, we have $\bm \beta_k^{t+1} = \tilde{\bm \beta}_k^t$ and $\bm \theta^{t+1} = \tilde{\bm \theta}^t$.

After discussing the updates of each sub problem and the Gaussian back substitution method, we summarize the entire algorithm. Given that the central machine does not store any data, we update the variables $\bm{\beta}$, $\bm{\theta}$, and $\bm{e}$, which do not require data loading, on the central machine. In contrast, the variables $\bm{\xi}_k$, $\bm{\beta}_k$, $\bm{b}_k$, and $\bm{d}_k$, which need to use the data set, are updated on the local machines. Meanwhile, to ensure the convergence of the algorithm, the variables $\bm{\xi}_k$, $\bm{\beta}$, and $\beta_0$ also need to be corrected according to (\ref{gaussian-total}). It should be noted that the updates of the dual variables $\bm{b}_k^{t + 1}$ and $\bm{d}_k^{t + 1}$ rely on $\bm{\beta}^{t + 1}$ and $\beta_0^{t + 1}$. This implies that if we follow the conventional update order of first updating the primal variables and then the dual variables, an additional round of communication between the central machine and the local machines will be incurred, thus increasing the communication cost. Fortunately, the update order of the dual variables has no impact on the accuracy and convergence of the algorithm. Therefore, we choose to update the dual variables $\bm{b}_k^t$ and $\bm{d}_k^t$ from the $t$-th iteration at the beginning of the $(t + 1)$-th iteration. That is, the update of the dual variables lags one step behind that of the primal variables. The updated $\bm{b}_k^t$ and $\bm{d}_k^t$ after adjustment can be expressed as follows:
\begin{subequations}
\label{dual}
\begin{align}
    \bm b_k^{t} & = \bm b_k^{t-1} - \mu(\tilde{\bm \xi_k}^{t-1} - \bm 1_k + \bar{\bm X_k}\bm \beta_k^t + \bm y_k \tilde{\beta_0}^{t-1}) \label{dual-b}
    \\
    \bm d_k^{t} & = \bm d_k^{t-1} - \mu(\bm \beta_k^t - \tilde{\bm \beta} ^{t-1}) \label{dual-d}
\end{align} 
\end{subequations}
Although updating the dual variable $\bm e$ does not incur additional communication costs, for the sake of algorithmic tidiness, we also choose to adjust $\bm e^t$.
\begin{equation}
    \label{dual-e}
    \bm e^t = \bm e^{t-1} - \mu(\bm \theta^t - \bm G \tilde{\bm \beta}^{t-1})
\end{equation}
We use Figure \ref{fig1} to illustrate the specific operations of the proposed algorithm, and the more detailed content of the algorithm is summarized in Algorithm \ref{alg1}.
\begin{figure}[H]
    \makebox[\textwidth][c]{\includegraphics[width=1.2\linewidth]{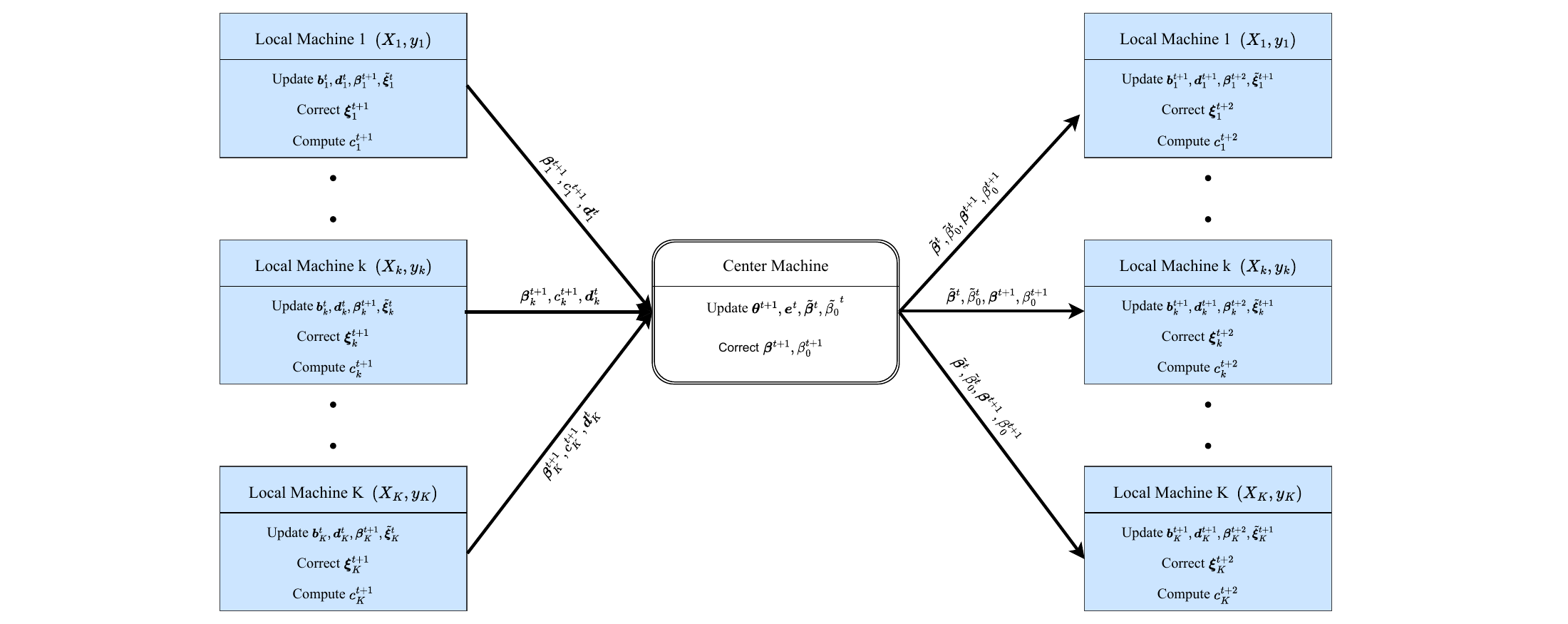}}
    \caption{Schematic diagram of the implementation of parallel ADMM algorithm}
    \label{fig1}
\end{figure}

\begin{algorithm}\small
\caption{\small{The dpADMM for solving the CR-SVMs}}
\label{alg1}
\begin{algorithmic}
\STATE {\textbf{Input:} $\bullet$ Central machine: $\mu, K, \lambda_1, \lambda_2,  \bm \beta^0, \bm \theta^0$ and $\bm e^0$.\\
\qquad \  \ \ \ \ $\bullet$ The $k$-th local machine:  $\boldsymbol{X}_k,\boldsymbol{y}_k$; $\mu, {\beta_0}_k^0, \bm \xi_k^0, \bm \beta_k^0$, $b_k^0$, $\bm d_k^0$ and selectable parameters $\tau \in (0,1]$ and $\delta>0$ for loss function.}
\STATE {\textbf{Output:} the total number of iterations $T$,  $\boldsymbol{\beta}^T$ and $\beta_0^T$. }
\STATE {\textbf{while} not converged \textbf{do}}

\STATE {\ \textbf{Local machines}: \ \  for $k =1 ,2, \dots, K$ (in parallel) \\
\qquad \qquad \qquad  \qquad \quad \ 1. Receive $\tilde{\bm \beta}^{t-1}, \tilde{\beta}_0^{t-1}$ and $\bm \beta^t, \beta_0^t$ transmitted by the central machine, \\
\qquad \qquad \qquad  \qquad \quad \ 2. Update $\bm b_k^{t}$ by (\ref{dual-b}) and $\bm d_k^{t}$ by (\ref{dual-d}),\\
\qquad \qquad \qquad  \qquad \quad \ 3. Update $\bm \beta_k^{t+1}$ by (\ref{betak-result}),\\
\qquad \qquad \qquad  \qquad \quad \ 4. Update $\tilde{\bm \xi}_k^t$ by Table \ref{tab4} according to different loss functions, then correct $\bm \xi_k ^{t + 1}$ by (\ref{gaussian-xi})\\
\qquad \qquad \qquad  \qquad \quad \ 5. Compute $c_k^{t+1}$ by (\ref{beta0k}), \\
\qquad \qquad \qquad  \qquad \quad \ 6. Send $c_k^{t+1}, \bm \beta_k^{t+1}$ and $\bm d_k^t$  to the central machine.

\STATE {\ \textbf{Central machine}: 1. Receive $\bm \beta_k^{t+1}, \bm d_k^t$ and $c_k^{t+1}$ transmitted by local machines $k$.\\ 
\qquad \qquad \qquad  \qquad \quad \ 2. Update $\bm \theta^{t+1}$ by Table \ref{tab5} according to different models, \\
\qquad \qquad \qquad  \qquad \quad \ 3. Update $\bm e^t$ by (\ref{dual-e}),\\
\qquad \qquad \qquad  \qquad \quad \ 4. Update $\tilde{\bm \beta}^t$ by Table \ref{tab3} according to different models, then correct $\bm \beta^{t+1}$ by (\ref{gaussian-beta}),\\
\qquad \qquad \qquad  \qquad \quad \ 5. Update $\tilde{\beta_0}^t$ by adding $K$ $c_k^{t+1}$, then correct $\beta_0^{t+1}$ by (\ref{gaussian-beta0}), \\
\qquad \qquad \qquad  \qquad \quad \ 6. Send $\tilde{\bm \beta}^t, \tilde{\beta_0}^t$ and $\bm \beta^{t+1}, \beta_0^{t+1}$ to the local machines.
}
}
\STATE {\textbf{end while}}
\STATE {\textbf{return} solution}.
\end{algorithmic}
\end{algorithm}
\subsection{Convergence and computational cost analysis}\label{sec3.4}
\vspace{-1em}
\quad \ The convergence analysis of the three-block ADMM algorithm modified by Gaussian back substitution has been comprehensively explored, as demonstrated in the works of \cite{He2012AlternatingDM, Wu2025ParallelAA}, among others. Here, we directly present the convergence conclusion of Algorithm \ref{alg1}. For the detailed proof, please refer to the  first section of supporting materials.
\begin{thm}
    \label{thm1}
    Let ${\bm h}^t = \left(\bm \beta^t, \beta_0^t, \{\bm \xi_k^t\}_{k=1}^K, \{\bm \beta_k^t\}_{k=1}^K, \bm \theta^t, \{\bm b_k^t\}_{k=1}^K, \{\bm d_k^t\}_{k=1}^K, \bm e^t \right)$ be generated by Algorithm \ref{alg1} with an initial feasible solution ${\bm h^0}$. The sequence ${\bm g}^t = \left(\{\bm \xi_k^t\}_{k=1}^K, \bm \beta^t, \beta_0^t,  \{\bm b_k^t\}_{k=1}^K, \{\bm d_k^t\}_{k=1}^K, \bm e^t \right)$ converge to ${\bm g}^*$, where ${\bm g}^* = \left(\{\bm \xi_k^*\}_{k=1}^K, \bm \beta^*, \beta_0^*,  \{\bm b_k^*\}_{k=1}^K, \{\bm d_k^*\}_{k=1}^K, \bm e^* \right)$ is an optimal solution point of (\ref{distributed structured sparse SVM}). For any positive integer \(T > 0\), the $O(1/T)$ convergence rate in a non-ergodic sense can also be obtained, i.e.,
    \begin{equation}
        \label{convergence}
        \|{\bm g}^T - {\bm g}^{T+1}\|_{\bm H}^2 \leq \frac{1}{c_0(T+1)}\|{\bm g}^0 - {\bm g}^*\|_{\bm H}^2
    \end{equation}
    where $c_0>0$ is a constant, 
   $\bm H$  is a positive definite matrix and its specific form can be found in the proof of this theorem. 
\end{thm}

In the proof of this theorem, we discover that the sequences $\{\| \bm{g}^{t} - \bm{g}^* \|_{\bm{H}}^2\}$ and $\{\| \bm{g}^{t+1} - \bm{g}^{t} \|_{\bm{H}}^2\}$ are monotonically non-increasing, and $\sum\limits_{t=0}^{\infty} c_0 \| (\bm{g}^{t+1} - {\bm{g}}^t) \|_{\bm{H}}^2 \leq \| \bm{g}^0 - \bm{g}^* \|_{\bm{H}}^2< +\infty$. From the lemma 1.1 in \cite{Deng2013ParallelMA}, we can draw the following conclusion.
\begin{cor}
    The improved sublinear convergence rate  in a non-ergodic sense can  be obtained, that is,
    \begin{align}
        \|\bm{g}^{t+1} - \bm{g}^{t} \|_{\bm{H}}^2 = o(1/t).
    \end{align}
\end{cor}

Next, we will provide the complexity of algorithm iteration, which is detailed  in the second section of supporting materials.
   \begin{thm}\label{Th2} 
The overall computational complexity of Algorithm \ref{alg1}  is  \begin{equation}\label{comp2}
\begin{cases}
  O(n_{\max}p^2) +  O(n_{\max}p) \times T  & \text{If } n_{\max} > p,\\
O(n_{\max}^2p) + O(p^2) \times T  & \text{otherwise},
\end{cases}
\end{equation}
where $n_{\max} = \max\{n_k \}_{k=1}^{K}$ and $T$  is the total number of iterations in the ADMM algorithm.
\end{thm}

This theorem indicates that adding local machines can effectively reduce the computational complexity of the algorithm. This is because increasing the number of local machines reduces the value of $n_{\max}$. Moreover, this complexity is independent of the loss function and regularization term, demonstrating the  universality of algorithm. 
\section{Numerical results}\label{sec4}
\quad \ To comprehensively evaluate the overall performance of the proposed parallel optimization framework and its algorithms, this chapter will conduct experiments across two dimensions: synthetic data and real-world data. The synthetic data experiments are designed to rigorously verify the algorithm’s convergence, computational efficiency, and its ability to learn from predefined data structures in a controlled environment. The real-world data experiments focus on the highly structurally complex task of music genre classification, aiming to test the algorithm’s generalization capability in practical scenarios and to deeply explore the musicological insights embedded in its analytical results.

In general, all the parameters for algorithms $\mu, \lambda_1$ and $\lambda_2$ need to be chosen by the cross-validation (CV) method, while leading to the high computational burden. To reduce the computational cost,  we  suggest that $\mu$ is selected from the set $\{0.01,0.1,1\}$. Moreover, to choose the optimal values for the regularization parameters $\lambda_1, \lambda_2$, we adopt the method proposed in \cite{Zhang2016VariableSF}, which minimizes the SVM information criterion (SVMIC). SVMIC is defined as:
\begin{align}
\text{SVMIC}(\bm \lambda) = \sum_{i = 1}^{n} \xi_n + |S_{\lambda_1, \lambda_2}| \log n + 2 \gamma \left(\begin{matrix}
p \\
|S_{\lambda_1, \lambda_2}|
\end{matrix}\right).
\end{align}
Here, $\gamma \in [0,1]$  and $|S_{\lambda_1, \lambda_2}|$ is the number of non-zero coordinates in estimator. All the experiments in this paper prove that the selection of these parameters is very effective. The iteration initial values of the prime and dual variables are both set to be $\bm 0$. 

In order to accelerate the convergence speed of the proposed parallel algorithm, we adopt the following acceleration method. One is the effective and simple method of adjusting $\mu$ proposed by \cite{Boyd2010DistributedOA}, that is
\begin{align}\label{sat}
\mu^{k+1}= \left\{ \begin{array}{l}
\mu^k*c_2, \  \text{if} \ c_1 \times \text{primal residual} < \text{dual residual},\\
\mu^k/c_2,\ \ \ \text{if} \ \text{primal residual} > c_1 \times \text{dual residual} ,\\
\mu^k, \qquad \ \text{otherwise},
\end{array} \right.
\end{align}
where $c_1=10, c_2=2$, the definitions of the above two residuals can be found in  Chapter 3.3 of \cite{Boyd2010DistributedOA}.  This self-adaptive tuning method has been used by \cite{Boyd2010DistributedOA}, and they also claimed that the convergence of ADMM can be guaranteed  if $\mu^k$ becomes fixed after a finite number of iterations. 
The proposed MLADMM algorithm is iterated until some stopping criterion is satisfied. We use the stopping criterion from the Chapter  3.3.1 of \cite{Boyd2010DistributedOA}. All experiments were conducted on a computer equipped with an AMD Ryzen 9 7950X 16-core processor (clocked at 4.50 GHz) and 32 GB of memory, using the R programming language. We put the R codes for implementing the proposed parallel algorithm and reproducing our experiments on  \url{https://github.com/xfwu1016/PADMM-for-Svms}. 
\subsection{Synthetic Data}\label{sec41}
\vspace{-1em}
\quad \
In this subsection,we focus on large-scale synthetic data experiments to systematically examine the core computational properties of the algorithm. The synthetic data are assumed to be generated from two normal distributions, $N(\mu_+, \Sigma)$ and $N(\mu_-, \Sigma)$, where the mean vectors are defined as $\mu_+ = (1, \ldots, 1, 0, \ldots, 0)^\top$ with the first 10 entries being 1 and the remaining entries 0, and $\mu_- = (-1, \ldots, -1, 0, \ldots, 0)^\top$ with the first 10 entries being -1 and the rest 0. The covariance matrix takes the block-diagonal form:
\begin{equation}
\boldsymbol{\Sigma}=\left[\begin{array}{*{20}{c}}
\bm {\Sigma}_{10}^*&  \bm0 \\ 
\bm 0 &\bm I_{(p-10)}
\end{array} \right].
\end{equation}
where $\Sigma^*_{10}$ is a $10 \times 10$ matrix whose diagonal entries are all 1 and off-diagonal entries are $\rho < 1$. Note that a larger $\rho$ indicates stronger correlation among the first 10 features. This data-generation scheme simulates a "sparse block structure" commonly encountered in high-dimensional classification problems: only the first 10 features truly contribute to the classification decision, and they exhibit a certain degree of correlation, while the remaining features are irrelevant noise. In our experiments, we fix $\rho = 0.5$. To further evaluate the robustness of the proposed algorithm, we inject label noise into the training data following the approach described in \cite{Wu2025MultiLA}. The labels of noise points are randomly selected from $\{+, -\}$ with equal probability, and their positions are generated from a Gaussian distribution $N(\mu_n, \Sigma_n)$ with $\mu_n = (0,0,\ldots,0)^\top$ and $\Sigma_n = \Sigma$. These noisy points primarily affect the labeling of samples near the decision boundary. The noise level is controlled by the proportion $\alpha\;(\%)$ of noisy samples in the training set. In this study, we set $\alpha = 20\%$. 
  
We employ our distributed parallel algorithm to implement EN‑SVM, SFL‑SVM, and SGL‑SVM. These implementations are then compared with existing non‑parallel algorithms \cite{Liang2024LinearizedAD, Wu2025MultiLA} for the respective models. Given the limited existing research on algorithms for the SGL‑SVM model, a direct comparison with a mature, established benchmark is currently infeasible. To address this, the performance evaluation in this study is conducted via a self‑comparison of the proposed algorithm under different configurations. Specifically, we compare its performance by varying the number of local machines $K$, setting it to $K = 1$ and $K = 5$, to assess the algorithm’s parallelization efficiency and scalability. For ease of reference, the distributed parallel ADMM algorithm proposed in this paper is denoted as DPADMM‑G, with the number of local machines $K=5$. The non‑distributed algorithm developed by \cite{Liang2024LinearizedAD} for the EN‑SVM is denoted as LADMM, while the non‑distributed algorithm proposed by \cite{Wu2025MultiLA} for the SFL‑SVM is denoted as MLADMM. We set $n=10000$ and $p=20000$.
\begin{table}[h]
    \centering
    \caption{Comparison of different algorithm performance for three SS-SVM Models }
    \renewcommand{\arraystretch}{1.5}
    \resizebox{1\columnwidth}{!}{
    \begin{tabular}{lcccccc}
    \hline
     & \multicolumn{2}{c}{EN-SVM} & \multicolumn{2}{c}{SFL-SVM} & \multicolumn{2}{c}{SGL-SVM}  \\
    \cmidrule(lr){2-3}\cmidrule(lr){4-5}\cmidrule(lr){6-7}
    Algorithm & LADMM & DPADMM-G (K=5) & MLADMM & DPADMM-G (K=5) & DPADMM-G (K=1) & DPADMM-G (K=5) \\
    \hline
    CAR & 0.920 & \textbf{0.931} & 0.919 & \textbf{0.924} & \textbf{0.923} & 0.906 \\
    CT & 66.37 & \textbf{20.98} & 71.33 & \textbf{25.36} & 78.36 & \textbf{23.25} \\
    NI & \textbf{49.17} & 68.25 & \textbf{57.85} & 72.08 & \textbf{55.36} & 73.68 \\
    NTSF & 10.02 & 10.03 & 10.05 & 10.04 & 10.02 & 10.07 \\
    \hline
    \end{tabular}}
    \label{tab12}
\end{table}

Table \ref{tab12} records the averaged results of 100 runs. Abbreviations in the table: CT for computation time, CAR for classification accuracy rate, NI for the number of iterations, and NTSF for the number of true selected features. As shown in Table \ref{tab12}, on the EN‑SVM and SFL‑SVM tasks, DPADMM‑G($K = 5$) slightly outperforms the corresponding non‑distributed algorithms in terms of CAR, while significantly reducing the CT by approximately $68\%–70\%$, demonstrating the efficiency of distributed computing. Although the NI of DPADMM‑G increases, its overall solving speed is still greatly improved because each iteration can be executed in parallel across multiple machines. In the SGL‑SVM task, compared with the $K = 1$ configuration, DPADMM‑G with $K = 5$ reduces the computation time by about $70\%$, further validating the favorable parallel scalability of the algorithm; its classification accuracy only experiences a slight decline, indicating that distributed computation does not compromise model accuracy under appropriate configurations. Regarding the accuracy of feature selection, the NTSF (Number of True Selected Features) values for all algorithms consistently fall within the range of $10.02–10.07$, aligning closely with the true number of effective features (10).

Given the limited research on distributed parallel algorithms for SS-SVM, it is challenging to identify a directly comparable method that addresses the same exact problem. To establish a meaningful performance benchmark, we adopt and adapt two representative existing studies. Following the distributed framework proposed by \cite{Guan2020AnEP} and denoted as QPADM-salck for non-convex SVMs, we apply it to solve an SVM model with the SCAD regularization. Similarly, leveraging two distributed algorithms developed by \cite{Wu2025AUC}, named QPADM-slack(GB) and M-QPADM-slack(GB), for high-dimensional quantile regression and classification, we adapt it to solve an SVM model with the SCAD-Pinball SVM. The performance of these two adapted baselines is then compared with that of our proposed distributed parallel algorithm, which is designed to solve a SFL-SVM with a SCAD variant. The number of training samples $n$ is set to $200,000$ and $500,000$, the number of testing samples $m$ is fixed at $1,000,000$, and the data dimension $p$ is chosen as $500$ and $1000$. 

\begin{table}[h]
\scriptsize 
\caption{Comparison of different Parallrl ADMM algorithms for SVMs}
\setlength{\belowcaptionskip}{0.2cm} 
\renewcommand{\arraystretch}{1.5}
\resizebox{1\columnwidth}{!}{
\begin{tabular}{lcccccccc}
\hline
  & \multicolumn{2}{l}{DPADMM-G} & \multicolumn{2}{r}{$(n,p)=(200000,500)$} & \multicolumn{2}{l}{DPADMM-G} & \multicolumn{2}{r}{$(n,p)=(500000,1000)$}\\ 
\cmidrule(lr){2-5}\cmidrule(lr){6-9}
$K$ & CAR &  CT   &  NI   &  NTSF  & CAR &  CT   &  NI   &  NTSF    \\ \hline
 5  & \textbf{0.982(0.002)} & \textbf{2.230(0.258)} & \textbf{19.412(1.822)} & \textbf{10.1(0.001)} & \textbf{0.979(0.002)} & \textbf{4.678(0.571)} & \textbf{25.121(2.803)} & \textbf{10.4(0.002)}  \\
10  & \textbf{0.980(0.003)} & \textbf{1.510(0.179)} & \textbf{22.513(2.132)} & \textbf{10.2(0.001)} & \textbf{0.975(0.003)} & \textbf{2.855(0.325)} & \textbf{30.346(2.965)} & \textbf{10.6(0.005)} \\
20 & \textbf{0.968(0.005)} & \textbf{0.984(0.121)} & \textbf{25.347(2.602)} & \textbf{10.7(0.003)} & \textbf{0.972(0.006)} & \textbf{1.942(0.282)} & \textbf{40.145(3.291)} & \textbf{10.8(0.005)}  \\\hline
   & \multicolumn{2}{l}{QPADM-slack}&\multicolumn{2}{r}{$(n,p)=(200000,500)$} & \multicolumn{2}{l}{QPADM-slack} & \multicolumn{2}{r}{$(n,p)=(500000,1000)$} \\  \cmidrule(lr){2-5}\cmidrule(lr){6-9}
$K$  & CAR &  CT   &  NI   &  NTSF  & CAR &  CT   &  NI   &  NTSF  \\\hline
 5  & 0.975(0.005) & {2.300(0.300)} & 20.012(2.100) & 10.4(0.004) & 0.972(0.005) & 4.750(0.650) & 30.721(3.100) & 10.7(0.005)  \\
10  & 0.973(0.006) & 1.580(0.220) & 23.113(2.400) & 10.5(0.004) & 0.969(0.006) & 2.930(0.390) & 35.946(3.250) & 10.9(0.008) \\
20 & 0.961(0.008) & 1.050(0.150) & 25.947(2.900) & 11.0(0.006) & 0.965(0.009) & 2.010(0.340) & 44.745(3.600) & 11.1(0.008)  \\\hline
   & \multicolumn{2}{l}{QPADM-slack(GB) } & \multicolumn{2}{r}{$(n,p)=(200000,500)$} & \multicolumn{2}{l}{QPADM-slack(GB)} & \multicolumn{2}{r}{$(n,p)=(500000,1000)$}\\   \cmidrule(lr){2-5}\cmidrule(lr){6-9}
$K$  & CAR &  CT   &  NI   &  NTSF   & CAR &  CT   &  NI   &  NTSF  \\\hline
 5  & 0.978(0.004) & {2.270(0.285)} & 19.812(2.000) & 10.3(0.003) & 0.975(0.004) & 4.720(0.620) & 30.521(3.000) & 10.6(0.004)  \\
10  & 0.976(0.005) & 1.550(0.205) & 22.913(2.300) & 10.4(0.003) & 0.971(0.005) & 2.900(0.370) & 35.746(3.150) & 10.8(0.007) \\
20 & 0.964(0.007) & 1.020(0.140) & 25.747(2.800) & 10.9(0.005) & 0.968(0.008) & 1.980(0.320) & 44.545(3.500) & 11.0(0.007) \\\hline
 & \multicolumn{2}{l}{M-QPADM-slack(GB)  } & \multicolumn{2}{r}{$(n,p)=(200000,500)$} & \multicolumn{2}{l}{M-QPADM-slack(GB)} & \multicolumn{2}{r}{$(n,p)=(500000,1000)$}\\  \cmidrule(lr){2-5}\cmidrule(lr){6-9}
$K$  &  CAR &  CT   &  NI   &  NTSF  & CAR &  CT   &  NI   &  NTSF  \\\hline
 5  & 0.980(0.003) & {2.250(0.270)} & 19.612(1.900) & 10.2(0.002) & 0.977(0.003) & 4.700(0.600) & 30.321(2.900) & 10.5(0.003)  \\
10  & 0.978(0.004) & 1.530(0.190) & 22.713(2.200) & 10.3(0.002) & 0.973(0.004) & 2.880(0.350) & 35.546(3.050) & 10.7(0.006) \\
20 & 0.966(0.006) & 1.000(0.130) & 25.547(2.700) & 10.8(0.004) & 0.970(0.007) & 1.960(0.300) & 44.345(3.400) & 10.9(0.006)  \\\hline
\end{tabular}}
\label{tab7}
\end{table}

From the Table \ref{tab7}, we can find that DPADMM-G exhibits the best overall performance, boasting the highest CAR, relatively low CT, fewer NI, and a stable NTSF close to 10. As the value of $K$ increases, the CAR generally declines for all algorithms, while the CT mostly decreases and the NI increases. As \cite{Wu2025ParallelAA}   and  \cite{Wu2025AUC} stated, this undesirable phenomenon does not result from the inherent defects of the algorithm itself. Instead, it stems from the consensus constraints in constructing the parallel computing results. This consensus structure leads to poorer algorithm performance as the number of local machines increases. When the data scale increases from $(n,p) = (200000,500)$ to $(n,p) = (500000,1000)$, the CAR of all algorithms slightly decreases, and both the CT and NI increase significantly, indicating that larger data volumes pose challenges to classification and computational performance. Although the PADMM algorithm proposed in this paper tends to deteriorate at large scales, it is more robust and scalable compared to other algorithms.
In Figure \ref{fig001}, we present the convergence of the proposed DPADMM-G under different numbers of local machines when the data size is $(n,p) = (500000,1000)$. Apparently, the proposed algorithm converges rapidly with different numbers of local machines. However, as the number of local machines rises, the convergence speed tends to slow down, which aligns with the increase in NI as the number of machines increases in Table \ref{tab7}. 
\begin{figure}[h]
    \centering
    \includegraphics[width=0.8\linewidth]{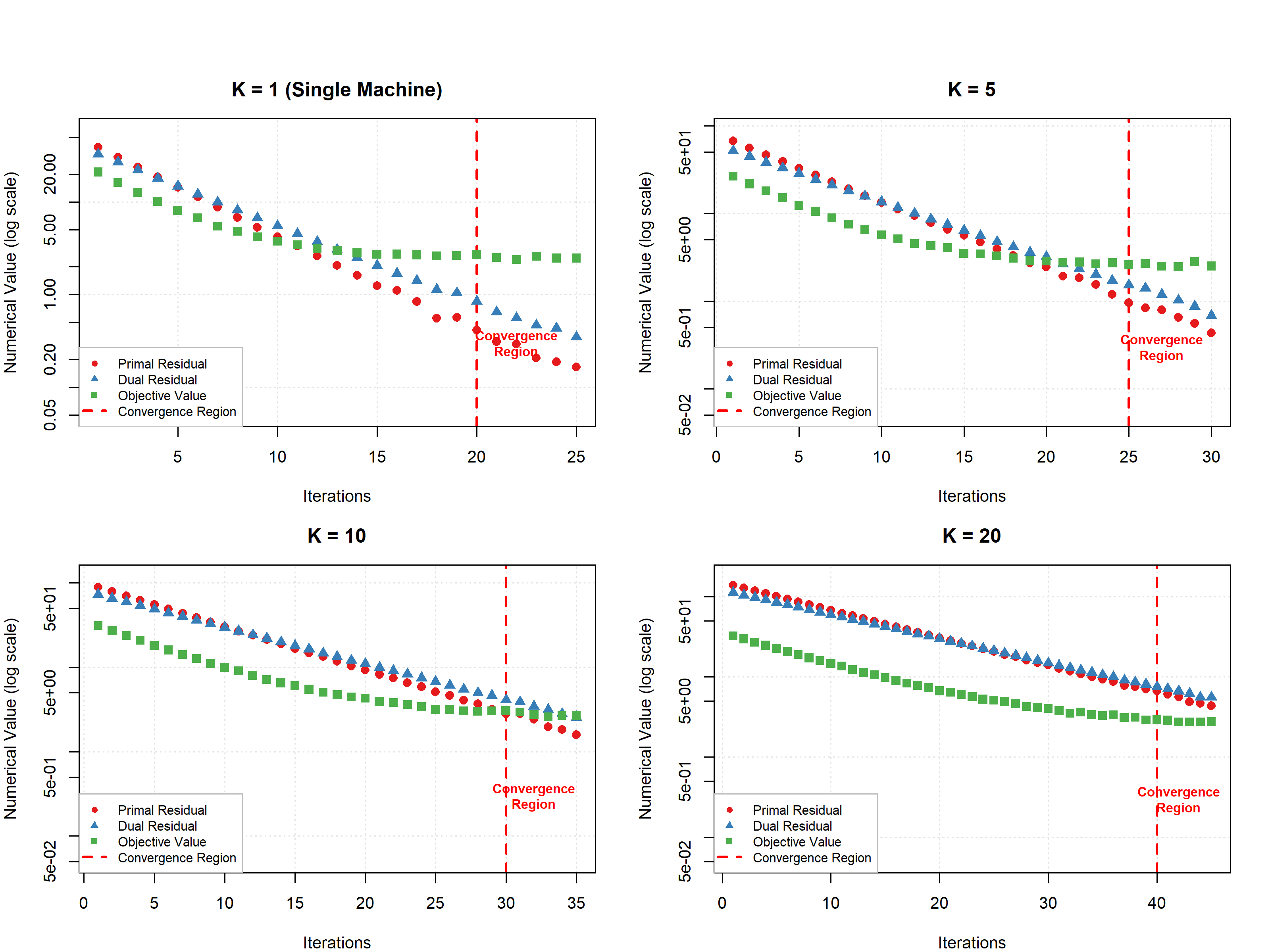}
    \caption{DPADMM-G convergence performance with different numbers of local machines.}
    \label{fig001}
\end{figure}

\begin{figure}[H]
    \centering
    \includegraphics[width=0.8\linewidth]{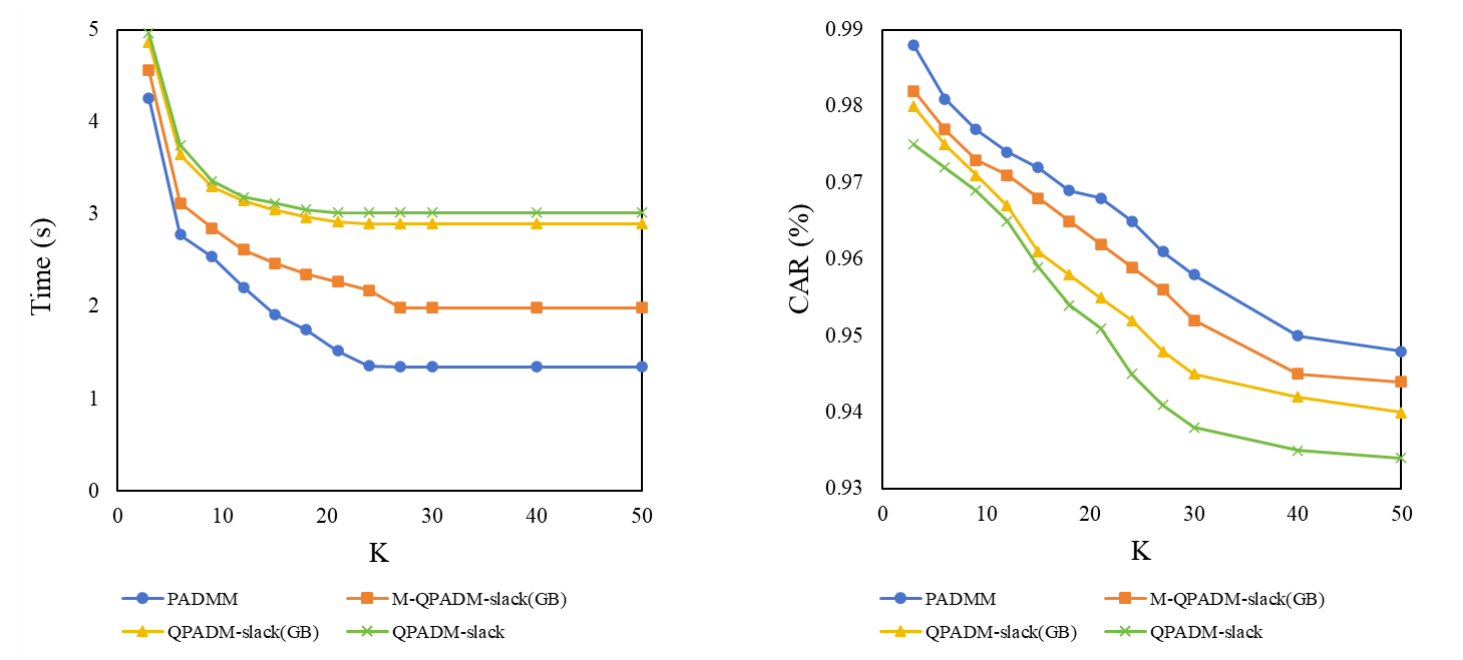}
    \caption{A schematic diagram depicting the variations of computation time and classification accuracy rate of four parallel algorithms with the number of local machines.}
    \label{fig:placeholder}
\end{figure}

In Table \ref{tab7}, only the cases with 5, 10 and 20 local machines are presented. The results of parallel computing with more local machines are shown in Figure \ref{fig:placeholder}. Figure \ref{fig:placeholder} reports the schematic diagrams of the changes in computation time and classification accuracy as $K$ increases. Regarding the computation time in Figure \ref{fig:placeholder}, when $K$ is greater than 20, the computation time tends to stabilize and no longer decreases. The reason for this is that the memory of the machine used in our experiment is limited, and it can only support the parallel operation of about 20 machines at a time. If there are more machines, due to insufficient memory, the acceleration effect cannot be achieved. As for CAR, it keeps decreasing as the number of machines increases. This is still due to the drawbacks of the consensus structure. When the number of local machines increases, the accuracy of the iterative solution will decline.
\subsection{Real Data}\label{sec42}
\vspace{-1em}
\quad \ To validate the value of the proposed algorithm in cross‑disciplinary scenarios, we apply it to the task of music genre classification, a problem that presents both computational challenges and musicological significance. Music audio features naturally exhibit a multi‑level structure: there are physical correlations among acoustic features, temporal smoothness, and semantic grouping . These structural characteristics align well with the proposed SS-SVM, providing an ideal platform to examine whether the algorithm can extract structured knowledge that aligns with musical cognition from complex data. We selected the publicly available benchmark dataset FMA (Free Music Archiv) in the field of music information retrieval as the verification platform. This selection was made after careful consideration for the following reasons: Firstly, the rich musical features provided by the FMA dataset inherently possess various structural information, which perfectly suits the application of the CR-SVMs model. Secondly, although the tracks in FMA may have multi-label attributes, through close collaboration with music domain experts, we determined an authoritative "primary genre" label for each piece of music, thus constructing a clear large-scale single-label music classification task. This task is not only large-scale enough to demonstrate the necessity of distributed computing, but its high-dimensional feature space and inherent structural characteristics also serve as an ideal testbed for examining whether the new algorithm can effectively perform sparse modeling and accurate classification in complex real-world scenarios. 

The FMA datase is a public music analysis dataset with various music features and metadata, which can be found in the GitHub repository (\url{https://github.com/mdeff/fma}) and comes in multiple versions. The small subset we commonly use, namely fma$\_$small, consists of 8 categories, with 1000 audio samples in each category, totaling 8000 audio samples. It should be noted that in the ``features.csv" of FMA, the feature names already include the feature types. We can group them by the prefixes of the feature names.  We group the 1036 features according to their physical meanings. Specifically,  we can divide the features into the following 7 groups: 16 time-domain features, including zcr\_mean, zcr\_var, rms\_mean, rmr\_var, etc.; 32 spectral features, including the mean and variance of spectral\_centroid, spectral\_bandwidth, spectral\_rolloff, spectral\_flatness, etc.; 160 MFCC features, including statistics such as the mean and variance of 20 MFCC coefficients; 96 chroma features, including statistics such as the mean and variance of 12 chroma features; 48 tonnetz features, including statistics such as the mean and variance of 6 tonnetz features; 518 echonest features, including a comprehensive feature set from the echonest audio analysis engine; and 166 other features, including rhythm features, beat features, and other audio features that are not classified into the above six main categories. 

We randomly selected two categories from the eight categories in fma$\_$small, with each category containing 1,000 samples. From each category, we randomly chose 800 samples as the training set and the remaining 200 samples as the test set. In the experiments, we primarily examined SGL-SVM, whose group‑sparse penalty enables feature selection at the group level, which helps produce more interpretable models. For comparison, we selected the SCAD‑SVM and SCAD‑Pinball SVM models, which are also suitable for high‑dimensional data, and solved them using the parallel algorithms proposed by \cite{Guan2020AnEP} and \cite{Wu2025ParallelAA}, respectively. Both were compared under the same distributed computing environment (with variable numbers of machines K). To better evaluate algorithm performance, we employed three evaluation metrics: “sparsity,” defined as the proportion of zero coefficients; “training accuracy,” which is the classification accuracy on the training set; and “test accuracy,” which is the classification accuracy on the test set. Additionally, we presented the results of running the algorithms in parallel on different numbers of local machines. Each experimental setting was independently simulated 100 times, and the average results are shown in Table \ref{tab43}.
\begin{table}[!ht]\footnotesize
    \centering
    \renewcommand{\arraystretch}{1.5}
    \caption{Comparative analysis of sparsity, training, and testing accuracies ($\%$) among PADMM, QPADM-slack and QPADM-slack(GB) algorithms}
    \begin{tabular}{llllllllll}
    \hline
          & \multicolumn{3}{c}{DPADMM-G} & \multicolumn{3}{c}{QPADM-slack(GB)} & \multicolumn{3}{c}{QPADM-slack}\\ 
        \cmidrule(lr){2 - 4}\cmidrule(lr){5 - 7}\cmidrule(lr){8 - 10}
        $K$ & Sparsity & Train & Test & Sparsity & Train & Test & Sparsity & Train & Test \\ \hline
        2 & \bf{90.38} & \bf{99.32} & \bf{97.27} & 85.99 & 95.14 & 92.96 & 83.64 & 93.38 & 92.00 \\ 
        4 & \bf{90.37} & \bf{99.26} & \bf{97.15} & 84.61 & 94.57 & 92.04 & 81.37 & 92.27 & 91.06 \\ 
        6 & \bf{90.36} & \bf{99.17} & \bf{97.08} & 83.60 & 93.89 & 91.48 & 80.14 & 91.78 & 90.60 \\ 
        8 & \bf{90.32} & \bf{99.04} & \bf{96.93} & 82.11 & 93.05 & 90.95 & 79.37 & 91.06 & 89.93 \\ 
        10 & \bf{90.13} & \bf{98.91} & \bf{96.84} & 81.37 & 92.58 & 90.09 & 77.99 & 90.78 & 89.09 \\ 
        12 & \bf{90.03} & \bf{98.79} & \bf{96.71} & 80.76 & 91.89 & 89.40 & 77.13 & 90.03 & 88.38 \\ 
        14 & \bf{89.87} & \bf{98.61} & \bf{96.59} & 79.37 & 91.17 & 88.66 & 76.43 & 89.74 & 87.84 \\ 
        16 & \bf{89.60} & \bf{98.50} & \bf{96.42} & 78.70 & 90.37 & 87.41 & 73.96 & 88.38 & 87.07 \\ 
        18 & \bf{89.67} & \bf{98.44} & \bf{96.38} & 77.66 & 89.80 & 86.29 & 72.66 & 87.79 & 86.26 \\ 
        20 & \bf{89.57} & \bf{98.34} & \bf{96.27} & 76.26 & 88.64 & 85.35 & 71.46 & 87.67 & 85.18 \\ \hline
    \end{tabular}
    \label{tab43}
\end{table}
Table \ref{tab43} presents the average performance of two algorithms under different parallelization scales. The results show that the DPADMM-G algorithm based on SGL-SVM consistently achieves significantly higher model sparsity while maintaining comparable or even superior classification accuracy. This phenomenon carries important practical implications: higher sparsity indicates that the model relies on fewer features for decision-making. More importantly, under the group-sparsity constraint, the retained features appear in "groups," which directly correspond to musical analysis dimensions such as "rhythm," "harmony," and "timbre." In contrast, the comparative approach based on SCAD-SVM, while also capable of achieving sparsity, performs feature selection in a fragmented manner, making it difficult to map the selected features to coherent musicological concepts.

Next, we report the influence of the previous seven pre-grouped sets on the accuracy of audio data classification. Regarding the coefficient estimation of SGL-SVM, the MFCC feature group has the largest number of non-zero estimated coefficients, followed by the spectral feature group, and then the chroma feature group. 
We have presented the sum of the absolute values of the estimated coefficients for each feature group and the schematic diagrams of the estimated coefficient distributions in Figures \ref{fig002} and \ref{fig004}.
In coefficient estimation, the number of non-zero coefficients also represents the importance of feature values. Based on music background knowledge, the MFCC feature group is the most important as it mimics human auditory perception and can effectively capture timbre and sound quality features. The spectral feature group comes next because it describes the spectral shape and energy distribution. The chroma feature group captures harmony and tonality information and also plays a relatively important role in music classification. Compared with these three groups, when the regularization penalty is relatively large, the other four feature groups have very few non-zero estimated coefficients, which can even be ignored. Therefore, when storage resources are limited, only these three feature groups can be collected to train a machine  learning model for audio classification. 
\begin{figure}[H]
    \centering
    \includegraphics[width=0.5\linewidth]{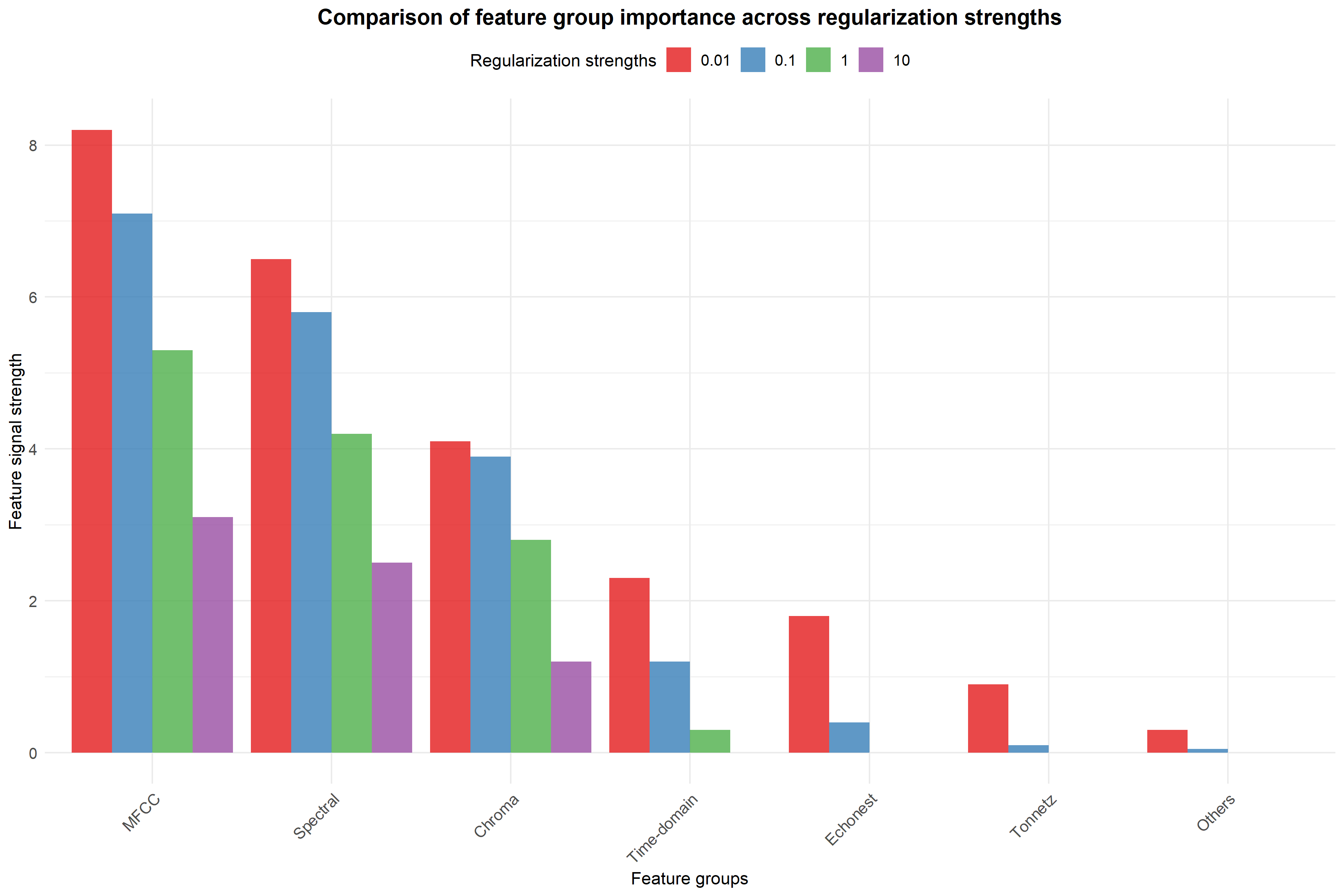}
    \caption{Group weight norms under different penalty parameters.}
    \label{fig002}
\end{figure}

The true value of a model extends beyond classification accuracy and lies in its ability to produce interpretable analysis results that align with music theory. Figure \ref{fig002} and \ref{fig004} show that the MFCC feature group contributes the largest weight, which is entirely consistent with research in music cognition science -- MFCCs simulate human auditory perception and are the core feature for perceiving timbre, which is the primary factor in distinguishing between different instruments and musical styles. The spectral features and chroma features follow closely, corresponding to "spectral brightness" and "harmonic content," respectively, both of which are key dimensions for describing musical style.
\begin{figure}[H]
    \centering
    \includegraphics[width=0.5\linewidth]{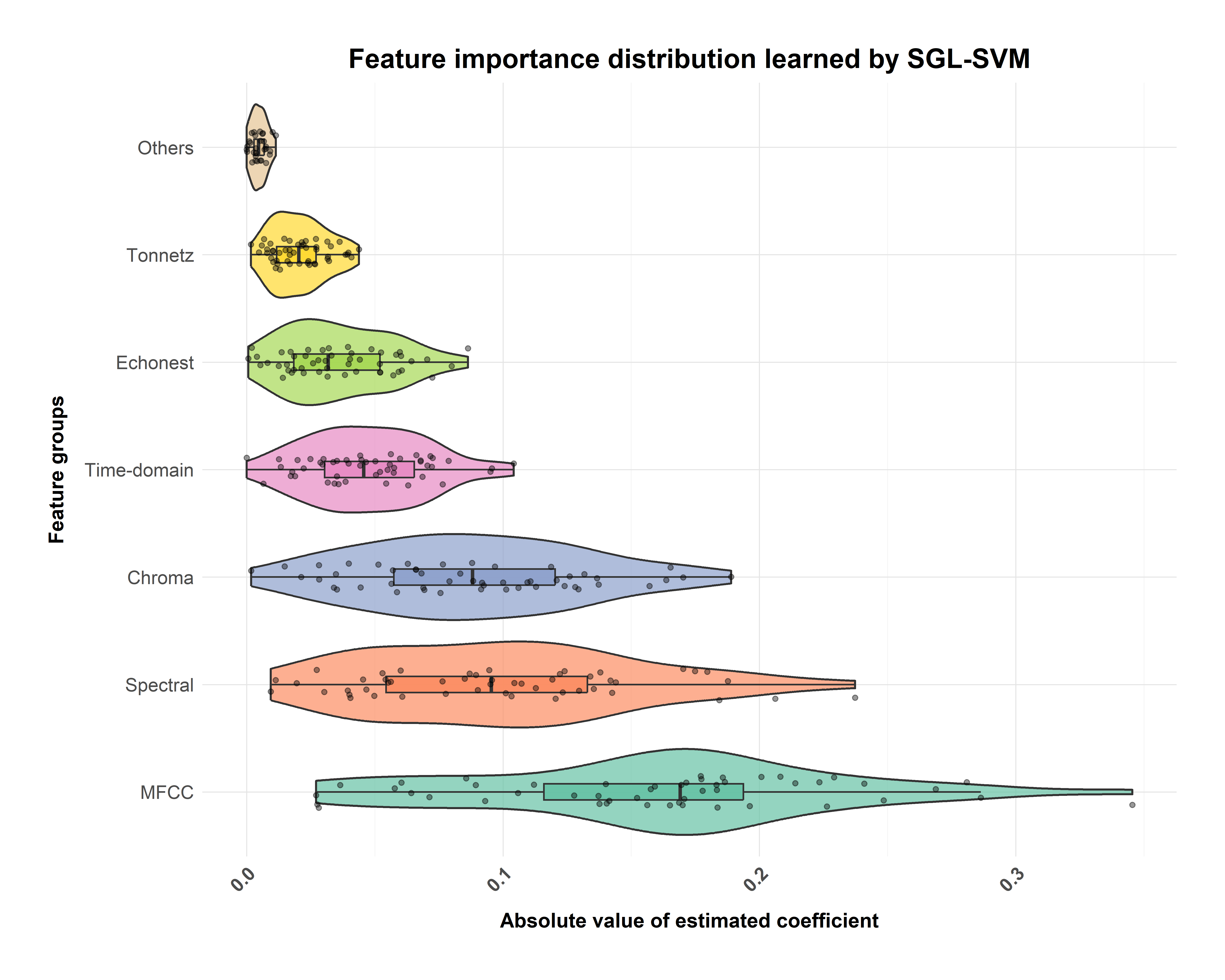}
    \caption{Feature group weight distribution.}
    \label{fig004}
\end{figure}
More critically, through structured sparse learning, the model automatically identifies these musically significant feature groups, while suppressing the weights of other feature groups, such as certain metadata features, to zero. This process essentially represents a "computational music analysis": the algorithm quantitatively confirms the foundational roles of dimensions such as "timbre $>$ harmony $>$ brightness" in genre discrimination from vast amounts of data, forming a quantitative echo of the qualitative descriptions found in traditional music textbooks.
\section{Conclusion}\label{sec5}
\quad \ This paper addresses the inadequacy of research on distributed parallel algorithms for SS-SVM by proposing a unified optimization framework based on convex and non-convex SVM problems. It derives a corresponding parallel ADMM algorithm and proves its global convergence with an improved sublinear rate. To ensure the comprehensiveness of the work, the paper also combines various loss functions discussed in \cite{Liang2024LinearizedAD} with a sparse group Lasso regularization term, resulting in a family of SGL-SVM as a significant complementary instance within the domain of SS-SVMs. Regardless of the regularization terms and SVM loss functions used, the computational complexity of our proposed algorithm remains consistent, which fully demonstrates its universality. Experiments on simulated data and the FMA dataset verify the reliability, stability, and scalability of the algorithm. Moreover, numerical experiments indicate that when training a machine learning model for music data classification, the MFCC feature group is the most crucial, followed by the spectral feature group and then the chroma feature group. 

Note that the algorithm proposed in this paper mainly focuses on convex loss functions. As for non-convex loss functions (refer to \cite{Wang2025SparseAR} and its references), we plan to solve them in future work. There are many other directions for further research in the future. At the theoretical level, future research will focus on enhancing the algorithm's efficiency and improving its theoretical foundation. Specifically, it includes investigating the algorithm's convergence properties under more challenging conditions such as, noisy or incomplete data and deriving more precise convergence rate estimates. Moreover, the algorithm's structure can be optimized and combined with other techniques, such as stochastic gradient descent or advanced acceleration methods \cite{Guo2020AnAF} to improve its convergence speed and overall performance. In terms of applications, we can extend the framework to multi-label tasks \cite{Zhang2021MultiTS}, such as music auto-tagging, to further test its broad effectiveness. 

\section*{CRediT authorship contribution statement}
Rongmei Liang: Writing–original draft, Software, Methodology, Data curation. Zizheng Liu: Writing–review and editing, Data curation, Funding acquisition. Xiaofei Wu: Validation, Supervision, Project administration.  Jingwen Tu: Funding acquisition, Visualization, Investigation.

\section*{Acknowledgements}
We are very grateful to Professor Bingsheng He for his valuable discussions with us, which greatly helps us use the  parallel ADMM algorithms to solve regression problems.  The research of  Liang and Wu was supported by the Scientific and Technological Research Program of Chongqing Municipal Education Commission [Grant Numbers KJQN202302003]. 
The research of Tu was supported by the General Program of Chongqing Natural Science Foundation [Grant Numbers CSTB2024NSCQ-MSX0077], the Doctoral Project of Chongqing Social Science Planning [Grant Numbers 2022BS064] and the Scientific and Technological Research Program of Chongqing Municipal Education Commission [Grant Numbers KJQN202301541].

\bibliographystyle{unsrt}
\bibliography{myrefq(abb)}

\end{document}